\documentclass{article}

\PassOptionsToPackage{numbers}{natbib}

     \usepackage[final]{neurips_2019}

\usepackage[utf8]{inputenc} 
\usepackage[T1]{fontenc}    
\usepackage{hyperref}       
\usepackage{url}            
\usepackage{booktabs}       
\usepackage{amsfonts}       
\usepackage{nicefrac}       
\usepackage{microtype}      
\usepackage{array}
\usepackage{multirow}
\usepackage{amsmath}
\usepackage{amssymb}
\usepackage{bm}
\usepackage{algorithm, algpseudocode}
\usepackage{color}
\usepackage{amsthm}
\usepackage{graphicx}
\usepackage{wrapfig}
\usepackage{caption}
\usepackage{subfigure}

\newtheorem{lemma}{Lemma}
\newtheorem{theorem}{Theorem}

\newtheorem{remark}{Remark}

\newcommand{\E}{\mathbb{E}}
\newcommand{\eat}[1]{}

\DeclareMathOperator*{\argmax}{arg\,max} 
\DeclareMathOperator*{\argmin}{arg\,min} 

\renewcommand{\algorithmicrequire}{\textbf{Input:}} 
\renewcommand{\algorithmicensure}{\textbf{Output:}} 
\let\oldReturn\Return
\renewcommand{\Return}{\State\oldReturn}

\makeatletter
\def\thanks#1{\protected@xdef\@thanks{\@thanks
        \protect\footnotetext{#1}}}
\makeatother

\title{Improving Black-box Adversarial Attacks with a Transfer-based Prior}

\author{%
  Shuyu Cheng$^*$, Yinpeng Dong$^*$, Tianyu Pang, Hang Su, Jun Zhu$^{\dagger}$\thanks{$^*$Equal contribution. $^{\dagger}$Corresponding author.} \\
  Dept. of Comp. Sci. and Tech., BNRist Center, State Key Lab for Intell. Tech. \& Sys.,\\
  Institute for AI, THBI Lab, Tsinghua University, Beijing, 100084, China\\
  \scriptsize \texttt{\{chengsy18, dyp17, pty17\}@mails.tsinghua.edu.cn, \{suhangss, dcszj\}@mail.tsinghua.edu.cn} \\
}

\begin{document}

\maketitle

\begin{abstract}
We consider the black-box adversarial setting, where the adversary has to generate adversarial perturbations without access to the target models to compute gradients. Previous methods tried to approximate the gradient either by using a transfer gradient of a surrogate white-box model, or based on the query feedback. However, these methods often suffer from low attack success rates or poor query efficiency since it is non-trivial to estimate the gradient in a high-dimensional space with limited information. To address these problems, we propose a prior-guided random gradient-free (P-RGF) method to improve black-box adversarial attacks, which takes the advantage of a transfer-based prior and the query information simultaneously. The transfer-based prior given by the gradient of a surrogate model is appropriately integrated into our algorithm by an optimal coefficient derived by a theoretical analysis. Extensive experiments demonstrate that our method requires much fewer queries to attack black-box models with higher success rates compared with the alternative state-of-the-art methods.
\end{abstract}

\section{Introduction}
Although deep neural networks (DNNs) have achieved significant success on various tasks~\cite{Goodfellow-et-al2016}, they have been shown to be vulnerable to adversarial examples~\cite{Biggio2013Evasion,Szegedy2013,Goodfellow2014}, which are crafted to fool the models by modifying normal examples with human imperceptible perturbations.
Many efforts have been devoted to studying the generation of adversarial examples, which is crucial to identify the weaknesses of deep learning algorithms~\cite{Szegedy2013,Athalye2018Obfuscated}, serve as a surrogate to evaluate robustness~\cite{carlini2016}, and consequently contribute to the design of robust deep learning models~\cite{madry2017towards}.

In general, adversarial attacks can be categorized into white-box attacks and black-box attacks. In the white-box setting, the adversary has full access to the model, and can use various gradient-based methods~\cite{Goodfellow2014,Kurakin2016,carlini2016,madry2017towards} to generate adversarial examples. 
In the more challenging black-box setting, the adversary has no or limited knowledge about the model, and crafts adversarial examples without any gradient information. The black-box setting is more practical in many real-world situations.

Many methods~\cite{papernot2016practical,Chen2017ZOO,Brendel2018Decision,Dong2017,ilyas2018black,Bhagoji_2018_ECCV,tu2018autozoom,ilyas2018prior,dong2019efficient} have been proposed to perform black-box adversarial attacks.
A common idea is to use an approximate gradient instead of the true gradient for crafting adversarial examples. The approximate gradient could be either the gradient of a surrogate model (termed as \textit{transfer-based} attacks) or numerically estimated by the zeroth-order optimization methods (termed as \textit{query-based} attacks).
In transfer-based attacks, adversarial examples generated for a different model are probable to remain adversarial for the target model due to the transferability~\cite{Papernot20162}. Although various methods~\cite{Dong2017,Liu2016,dong2019evading} have been introduced to improve the transferability, the attack success rate is still unsatisfactory. The reason is that there lacks an adjustment procedure in transfer-based attacks when the gradient of the surrogate model points to a non-adversarial region of the target model.
In query-based adversarial attacks, the gradient can be estimated by various methods, such as finite difference~\cite{Chen2017ZOO,Bhagoji_2018_ECCV}, random gradient estimation~\cite{tu2018autozoom}, and natural evolution strategy~\cite{ilyas2018black}. These methods usually result in a higher attack success rate compared with the transfer-based attack methods~\cite{Chen2017ZOO,Bhagoji_2018_ECCV}, but they require a tremendous number of queries to perform a successful attack.
The inefficiency mainly comes from the underutilization of priors, since the current methods are nearly optimal to estimate the gradient~\cite{ilyas2018prior}.


To address the aforementioned problems and improve black-box attacks, we propose a \textbf{prior-guided random gradient-free (P-RGF)} method to utilize the transfer-based prior for query-efficient black-box attacks under the gradient estimation framework.
The transfer-based prior is given by the gradient of a surrogate white-box model, which contains abundant prior knowledge of the true gradient.
Our method provides a gradient estimate by querying the target model with random samples that are biased towards the transfer gradient and acquiring the corresponding loss values.
We provide a theoretical analysis on deriving the optimal coefficient, which controls the strength of the transfer gradient.
Our method is also flexible to integrate other forms of prior information. As a concrete example, we incorporate the commonly used \emph{data-dependent prior}~\cite{ilyas2018prior} into our algorithm along with the transfer-based prior. 
Extensive experiments demonstrate that  
our method significantly outperforms the previous state-of-the-art methods in terms of black-box attack success rate and query efficiency, which verifies the superiority of our method for black-box adversarial attacks.

\section{Background} 
In this section, we review the background and the related work on black-box adversarial attacks.

\vspace{-1ex}
\subsection{Adversarial setup}
\vspace{-1ex}
Given a classifier $C(x)$ and an input-label pair $(x, y)$, the goal of attacks is to generate an adversarial example $x^{adv}$ that is misclassified while the distance between the adversarial input and the normal input measured by the $\ell_p$ norm is smaller than a preset threshold $\epsilon$ as 
\begin{equation}
   C(x^{adv}) \neq y, \text{ s.t. }  \|x^{adv}-x\|_p\leq\epsilon.
\end{equation}
Note that this corresponds to the untargeted attack. We present our framework and algorithm based on the untargeted attack for clarity, while the extension to the targeted one is straightforward.

An adversarial example can be generated by solving the constrained optimization problem as 
\begin{equation}
    x^{adv} = \argmax_{x':\|x' - x\|_{p} \leq \epsilon} f(x', y),
    \label{eq:problem}
\end{equation}
where $f$ is a loss function on top of the classifier $C(x)$, e.g., the cross-entropy loss.
Many gradient-based methods~\cite{Goodfellow2014,Kurakin2016,carlini2016,madry2017towards} have been proposed to solve this optimization problem.
The state-of-the-art projected gradient descent (PGD)~\cite{madry2017towards} iteratively generates adversarial examples as 
\begin{equation}
x_{t+1}^{adv} = \Pi_{\mathcal{B}_p(x,\epsilon)} (x_t^{adv} + \eta\cdot g_t),
\label{eq:iter}
\end{equation}
where $\Pi$ is the projection operation, $\mathcal{B}_p(x,\epsilon)$ is the $\ell_p$ ball centered at $x$ with radius $\epsilon$, $\eta$ is the step size, and $g_t$ is the normalized gradient under the $\ell_p$ norm, e.g., $g_t = \frac{\nabla_{x}f(x_t^{adv},y)}{\|\nabla_{x}f(x_t^{adv},y)\|_2}$ under the $\ell_2$ norm, and $g_t = \mathrm{sign}(\nabla_{x}f(x_t^{adv},y))$ under the $\ell_{\infty}$ norm.
This method requires full access to the gradient of the target model, which is designed under the white-box attack setting.

\vspace{-1ex}
\subsection{Black-box attacks}
\vspace{-1ex}
The direct access to the model gradient is unrealistic in many real-world applications, where we need to perform attacks in the black-box manner.
We can still adopt the PGD method to generate adversarial examples, except that the true gradient $\nabla_{x}f(x,y)$ is usually replaced by an approximate gradient. Black-box attacks can be roughly divided into transfer-based attacks and query-based attacks. 
Transfer-based attacks adopt the gradient of a surrogate white-box model to generate adversarial examples, which are probable to fool the black-box model due to the transferability~\cite{papernot2016practical,Liu2016,Dong2017}. Query-based attacks estimate the gradient by the zeroth-order optimization methods, when the loss values could be accessed through queries. 
\citet{Chen2017ZOO} propose to use the symmetric difference quotient~\cite{lax2014calculus} to estimate the gradient at each coordinate as
\begin{equation}
\hat{g}_i = \frac{f(x+\sigma e_i,y)-f(x-\sigma e_i,y)}{2\sigma} \approx \frac{\partial f(x,y)}{\partial x_i},
\label{eq:finite-difference}
\end{equation}
where $\sigma$ is a small constant, and $e_i$ is the $i$-th unit basis vector.
Although query-efficient mechanisms have been developed~\cite{Chen2017ZOO,Bhagoji_2018_ECCV}, the coordinate-wise gradient estimation inherently results in the query complexity being proportional to the input dimension $D$, which is prohibitively large with high-dimensional input space, e.g., $D\approx270$,$000$ for ImageNet~\cite{russakovsky2015imagenet}.
To improve query efficiency, the approximated gradient $\hat{g}$ can be estimated by the random gradient-free (RGF) method~\cite{nesterov2017random,ghadimi2013stochastic,duchi2015optimal} as
\begin{equation}
     \hat{g} = \frac{1}{q}\sum_{i=1}^{q}\hat{g}_i, \text{ where  } \hat{g}_i = \frac{f(x+\sigma u_i,y)-f(x,y)}{\sigma}\cdot u_i,
\label{eq:estimate}
\end{equation}
where $\{u_i\}_{i=1}^q$ are the random vectors independently sampled from a distribution $\mathcal{P}$ on $\mathbb{R}^D$, and $\sigma$ is the parameter to control the sampling variance. It is noted that $\hat{g}_i \rightarrow u_i^\top \nabla_x f(x,y) \cdot u_i$ when $\sigma\rightarrow 0$, which is nearly an unbiased estimator of the gradient~\cite{duchi2015optimal} when $\mathbb{E}[u_iu_i^\top] = \mathbf{I}$. $\hat{g}$ is the average estimation over $q$ random directions to reduce the variance.
The natural evolution strategy (NES)~\cite{ilyas2018black} is another variant of Eq.~\eqref{eq:estimate}, which conducts the antithetic sampling over a Gaussian distribution.
\citet{ilyas2018prior} show that these methods are nearly optimal to estimate the gradient, but their query efficiency could be improved by incorporating informative priors.
They identify the time and data-dependent priors for black-box attacks. Different from the alternative methods, our proposed transfer-based prior is more effective as shown in the experiments. Moreover, the transfer-based prior can also be used together with other priors. We demonstrate the flexibility of our algorithm by incorporating the commonly used data-dependent prior as an example.

\subsection{Black-box attacks based on both transferability and queries}
There are also several works that adopt both the transferability of adversarial examples and the model queries for black-box attacks.
\citet{papernot2016practical,Papernot20162} train a local substitute model to mimic the black-box model with a synthetic dataset, in which the labels are given by the black-box model through queries. Then the black-box model is attacked by the adversarial examples generated for the substitute model based on the transferability.
A meta-model~\cite{oh2017towards} can reverse-engineer the black-box model and predict its attributes (such as architecture, optimization procedure, and training data) through a sequence of queries.
Given the predicted attributes of the black-box model, the attacker can find similar surrogate models, which are better to craft transferable adversarial examples against the black-box model.
These methods all use queries to obtain knowledge of the black-box model, and train/find surrogate models to generate adversarial examples, with the purpose of improving the transferability.
However, we do not optimize the surrogate model, but focus on utilizing the gradient of a fixed surrogate model to obtain a more accurate gradient estimate.

A recent work~\cite{brunner2018guessing} also uses the gradient of a surrogate model to improve the efficiency of query-based black-box attacks. This method focuses on a different attack scenario, where the model only provides the hard-label outputs, but we consider the setting where the loss values could be accessed. Moreover, this method controls the strength of the transfer gradient by a preset hyperparameter, but we obtain its optimal value through a theoretical analysis based on the gradient estimation framework.
It's worth mentioning that a similar but independent work~\cite{maheswaranathan2018guided} also uses surrogate gradients to improve zeroth-order optimization, but they did not apply their method to black-box adversarial attacks.

\section{Methodology}
In this section, we first introduce the gradient estimation framework. Then we propose the prior-guided random gradient-free (P-RGF) algorithm. We further incorporate the data-dependent prior~\cite{ilyas2018prior} into our algorithm. We also provide an alternative algorithm for the same purpose in Appendix~\ref{sec:average}.

\subsection{Gradient estimation framework}
\label{sec:framework}

The key challenge in black-box adversarial attacks is to estimate the gradient of a model, which can be used to conduct gradient-based attacks.
In this paper, we aim to estimate the gradient $\nabla_{x}f(x,y)$ of the black-box model $f$
more accurately to improve black-box attacks.
We denote the gradient $\nabla_x f(x,y)$ by $\nabla f(x)$ in the following for simplicity. We assume that $\nabla f(x)\neq 0$ in this paper.
The objective of gradient estimation is to find the best estimator, which approximates the true gradient $\nabla f(x)$ by reaching the minimum value of the loss function as
\begin{equation}
    \hat{g}^* = \argmin_{\hat{g}\in \mathcal{G}} L(\hat{g}), 
\label{eq:ge_problem}
\end{equation}
where $\hat{g}$ is a gradient estimator given by any estimation algorithm,  $\mathcal{G}$ is the set of all possible gradient estimators, and $L(\hat{g})$ is a loss function to measure the performance of the estimator $\hat{g}$.
Specifically, we let the loss function of the gradient estimator $\hat{g}$ be
\begin{equation}
    L(\hat{g}) = \min_{b\geq0} \E\|\nabla f(x)-b\hat{g}\|_2^2,
\label{eq:loss}
\end{equation}
where the expectation is taken over the randomness of the estimation algorithm to obtain $\hat{g}$. 
The loss $L(\hat{g})$ is the minimum expected squared $\ell_2$ distance between the true gradient $\nabla f(x)$ and scaled estimator $b\hat{g}$. 
The previous work~\cite{tu2018autozoom} also uses the expected squared $\ell_2$ distance $\E\|\nabla f(x)-\hat{g}\|_2^2$ as the loss function, which is similar to ours. However, the value of this loss function will change with different magnitude of the estimator $\hat{g}$. In generating adversarial examples, the gradient is usually normalized~\cite{Goodfellow2014,madry2017towards}, such that the direction of the gradient estimator, instead of the magnitude, will affect the performance of attacks. Thus, we incorporate a scaling factor $b$ in Eq.~\eqref{eq:loss} and minimize the error w.r.t. $b$, which can neglect the impact of the magnitude on the loss of the estimator $\hat{g}$.

\subsection{Prior-guided random gradient-free method}
\label{sec:biased_sampling}
In this section, we present the prior-guided random gradient-free (P-RGF) method, which is a variant of the random gradient-free (RGF) method. Recall that in RGF, the gradient can be estimated via a set of random vectors $\{u_i\}_{i=1}^q$ as in Eq.~\eqref{eq:estimate} with $q$ being the number of random vectors.
Directly using RGF without prior information will result in poor query efficiency as shown in our experiments.
In our method, we propose to sample the random vectors that are biased towards the transfer gradient, to fully exploit the prior information.

Let $v$ be the normalized transfer gradient of a surrogate model such that $\|v\|_2=1$, and the cosine similarity between the transfer gradient and the true gradient be 
\begin{equation}
    \alpha=v^\top\overline{\nabla f(x)} \text{  with   } \overline{\nabla f(x)}=\|\nabla f(x)\|_2^{-1}\nabla f(x), 
\end{equation}
where $\overline{\nabla f(x)}$ is the $\ell_2$ normalization of the true gradient $\nabla f(x)$.\footnote{ 
We will use $\overline{e}$ to denote the $\ell_2$ normalization of a vector $e$ in this paper.}
We assume that $\alpha\geq 0$ without loss of generality, since we can reassign  $v\leftarrow -v$ when $\alpha<0$.
Although the true value of $\alpha$ is unknown, we could estimate it efficiently, which will be introduced in Sec.~\ref{sec:alpha}.

For the RGF estimator $\hat{g}$ in Eq.~\eqref{eq:estimate}, we further assume that the sampling distribution $\mathcal{P}$ is defined on the unit hypersphere in the $D$-dimensional space, such that the random vectors $\{u_i\}_{i=1}^q$ drawn from $\mathcal{P}$ satisfy $\|u_i\|_2=1$. Then, we can represent the loss of the RGF estimator by the following theorem.
\begin{theorem}\label{the:1}
(Proof in Appendix \ref{sec:proof_the-1}) If $f$ is differentiable at $x$, the loss of the RGF estimator $\hat{g}$ is
\begin{equation}
    \lim_{\sigma\to 0}L(\hat{g})=\|\nabla f(x)\|_2^2-\frac{\big(\nabla f(x)^\top\mathbf{C}\nabla f(x)\big)^2}{(1-\frac{1}{q})\nabla f(x)^\top\mathbf{C}^2\nabla f(x)+\frac{1}{q}\nabla f(x)^\top\mathbf{C}\nabla f(x)},
    \label{eq:biased_obj}
\end{equation}
where $\sigma$ is the sampling variance, $\mathbf{C}=\E[u_iu_i^\top]$ with $u_i$ being the random vector, $\|u_i\|_2=1$, and $q$ is the number of random vectors as in Eq.~\eqref{eq:estimate}.  
\end{theorem}

Given the definition of $\mathbf{C}$, it needs to satisfy two constraints: (1) it should be positive semi-definite; (2) its trace should be $1$ since $\mathrm{Tr}(\mathbf{C})=\E[\mathrm{Tr}(u_i u_i^\top)]=\E[u_i^\top u_i]=1$.
It is noted from \autoref{the:1} that we can minimize $L(\hat{g})$ by optimizing $\mathbf{C}$, i.e., we can achieve an optimal gradient estimator by carefully sampling the random vectors $u_i$, yielding an query-efficient adversarial attack.

Specifically, $\mathbf{C}$ can be decomposed as $\sum_{i=1}^D \lambda_i v_i v_i^\top$, where $\{\lambda_i\}_{i=1}^D$ and $\{v_i\}_{i=1}^D$ are the eigenvalues and orthonormal eigenvectors of $\mathbf{C}$, and $\sum_{i=1}^D \lambda_i=1$.
In our method, since we propose to bias $u_i$ towards $v$ to exploit its prior information, we can specify an eigenvector to be $v$, and let the corresponding eigenvalue be a tunable coefficient.
For the other eigenvalues, we set them to be equal since we do not have any prior knowledge about the other eigenvectors.
In this case, we let
\begin{equation}
    \mathbf{C}=\lambda v v^\top + \frac{1-\lambda}{D-1}(\mathbf{I}-v v^\top),
    \label{eq:mix_C}
\end{equation}
where $\lambda \in [0,1]$ controls the strength of the transfer gradient that the random vectors $\{u_i\}_{i=1}^q$ are biased towards.
We can easily construct a random vector with unit length while satisfying Eq.~\eqref{eq:mix_C} (proof in Appendix \ref{sec:proof_eq-sample-rv}) as
\begin{equation}
    u_i=\sqrt{\lambda}\cdot v+\sqrt{1-\lambda}\cdot\overline{(\mathbf{I}-vv^\top)\xi_i},
    \label{eq:sample_rv}
\end{equation}
where $\xi_i$ is sampled uniformly from the unit hypersphere. Hereby, the problem turns to optimizing $\lambda$ that minimizes $L(\hat{g})$. 
The previous work~\cite{tu2018autozoom} can also be categorized as a special case of our method when $\lambda=\frac{1}{D}$ and  $\mathbf{C}=\frac{1}{D}\mathbf{I}$, such that the random vectors are drawn from the uniform distribution on the hypersphere.  
When $\lambda\in[0,\frac{1}{D})$, it indicates that the transfer gradient is worse than a random vector, so we are encouraged to search in other directions by using a small $\lambda$.
To find the optimal $\lambda$, we plug Eq.~\eqref{eq:mix_C} into Eq.~\eqref{eq:biased_obj}, and obtain the closed-form solution (proof in Appendix \ref{sec:proof_eq-lambda}) as
\begin{align}
\small
    \lambda^* =
    \begin{cases}
        \hfil 0 & \text{if} \; \alpha^2\leq\dfrac{1}{D+2q-2}  \\
        \dfrac{(1-\alpha^2)(\alpha^2(D+2q-2)-1)}{2\alpha^2 Dq - \alpha^4 D(D+2q-2) - 1} & \text{if} \; \dfrac{1}{D+2q-2} < \alpha^2 < \dfrac{2q-1}{D+2q-2} \\
        \hfil 1 & \text{if} \; \alpha^2 \geq \dfrac{2q-1}{D+2q-2}
    \end{cases}.
    \label{eq:lambda-1}
\end{align}
\textbf{Remark.} It can be proven (in Appendix \ref{sec:proof_monotonicity-lambda}) that $\lambda^*$ is a monotonically increasing function of $\alpha^2$, and a monotonically decreasing function of $q$ (when $\alpha^2>\frac{1}{D}$).
It means that a larger $\alpha$ or a smaller $q$ (when the transfer gradient is not worse than a random vector) would result in a larger $\lambda^*$, which makes sense since we tend to rely on the transfer gradient more when (1) it approximates the true gradient better; (2) the number of queries is not enough to provide much gradient information.

\begin{algorithm}[t]
\small
\caption{Prior-guided random gradient-free (P-RGF) method}
\label{alg:biased}
\begin{algorithmic}[1]
\Require The black-box model $f$; input $x$ and label $y$; the normalized transfer gradient $v$; sampling variance $\sigma$; number of queries $q$; input dimension $D$.
\Ensure Estimate of the gradient $\nabla f(x)$.
\State Estimate the cosine similarity $\alpha=v^\top\overline{\nabla f(x)}$ (detailed in Sec.~\ref{sec:alpha});
\State Calculate $\lambda^*$ according to Eq.~\eqref{eq:lambda-1} given $\alpha$, $q$, and $D$;
\If {$\lambda^*=1$}
\Return $v$;
\EndIf
\State $\hat{g} \leftarrow \mathbf{0}$;
\For {$i = 1$ to $q$}
\State Sample $\xi_i$ from the uniform distribution on the $D$-dimensional unit hypersphere;
\State $u_i=\sqrt{\lambda^*}\cdot v+\sqrt{1-\lambda^*}\cdot\overline{(\mathbf{I}-vv^\top)\xi_i}$;

\State $\hat{g} \leftarrow \hat{g} + \dfrac{f(x + \sigma u_i,y) - f(x,y)}{\sigma} \cdot u_i$;
\EndFor
\Return $\nabla f(x)\leftarrow \dfrac{1}{q}\hat{g}$.
\end{algorithmic}
\end{algorithm}

We summarize the P-RGF method in Algorithm~\ref{alg:biased}.
Note that when $\lambda^*=1$, we directly return the transfer gradient as the estimate of $\nabla f(x)$, which can save many queries.

\subsection{Estimation of cosine similarity}
\label{sec:alpha}
To complete our algorithm, we also need to estimate $\alpha=v^\top\overline{\nabla f(x)} =\frac{v^\top\nabla f(x)}{\|\nabla f(x)\|_2}$, where $v$ is the normalized transfer gradient.
Note that the inner product $v^\top\nabla f(x)$ can be easily estimated by the finite difference method
\begin{equation}
    v^\top\nabla f(x) \approx \frac{f(x+\sigma v, y)-f(x,y)}{\sigma}, 
    \label{eq:inner-product}
\end{equation}
using a small $\sigma$. Hence, the problem is reduced to estimating $\|\nabla f(x)\|_2$.


Suppose that it is allowed to conduct $S$ queries to estimate $\|\nabla f(x)\|_2$. We first draw a different set of $S$ random vectors $\{w_s\}_{s=1}^S$ independently and uniformly from the $D$-dimensional unit hypersphere, and then estimate $w_s^\top\nabla f(x)$ using Eq.~\eqref{eq:inner-product}. Suppose that we have a $r$-degree homogeneous function $g$ of $S$ variables, i.e., $g(az)=a^r g(z)$ where $a\in\mathbb{R}$ and $z\in \mathbb{R}^S$, then we have
\begin{align}
\label{norm_est}
    g\big(\mathbf{W}^\top\nabla f(x)\big)=\|\nabla f(x)\|_2^r \cdot g\big(\mathbf{W}^\top\overline{\nabla f(x)}\big), 
\end{align}
where $\mathbf{W}$ is the collection of the random vectors as $\mathbf{W}=[w_1, ..., w_S]$. In this case, the norm of the gradient $\|\nabla f(x)\|_2$ could be computed easily if both $g\big(\mathbf{W}^\top\nabla f(x)\big)$ and $g\big(\mathbf{W}^\top\overline{\nabla f(x)}\big)$ are available. 
Note that $g\big(\mathbf{W}^\top\nabla f(x)\big)$ can be calculated since each $w_s^\top\nabla f(x)$ is available. 

However, it is non-trivial to obtain the value of $w_s^\top\overline{\nabla f(x)}$ as well as the function value $g\big(\mathbf{W}^\top\overline{\nabla f(x)}\big)$.
Nevertheless, we note that the distribution of $w_s^\top\overline{\nabla f(x)}$ is the same regardless of the direction of $\overline{\nabla f(x)}$, thus we can compute the expectation of the function value $\E\big[g\big(\mathbf{W}^\top\overline{\nabla f(x)}\big)\big]$.
Based on that, we use $\frac{g(\mathbf{W}^\top\nabla f(x))}{\E[g(\mathbf{W}^\top\overline{\nabla f(x)})]}$ as an unbiased estimator of $\|\nabla f(x)\|_2^r$. In particular, we choose $g$ as $g(z)=\frac{1}{S}\sum_{s=1}^S z_s^2$. Then $r=2$, and we have 
\begin{align}
\label{eq:g}
    \E\big[g\big(\mathbf{W}^\top\overline{\nabla f(x)}\big)\big] = \E \big[(w_1^\top\overline{\nabla f(x)})^2]=\overline{\nabla f(x)}^\top \E[w_1w_1^\top]\overline{\nabla f(x)}=\frac{1}{D}.
\end{align}
By plugging Eq.~\eqref{eq:g} into Eq.~\eqref{norm_est}, we can estimate the gradient norm by
\begin{equation}
    \|\nabla f(x)\|_2 \approx \sqrt{\frac{D}{S}\sum_{s=1}^S (w_s^\top\nabla f(x))^2} \approx \sqrt{\frac{D}{S}\sum_{s=1}^S \Big(\frac{f(x+\sigma w_s,y) - f(x,y)}{\sigma}\Big)^2}.
\end{equation}


To save queries, we estimate the gradient norm periodically instead of in every iteration, since usually it does not change very fast in the optimization process.

\subsection{Incorporating the data-dependent prior}
\label{sec:dp}

The proposed P-RGF method is generally flexible to integrate other priors. As a concrete example, we incorporate the commonly used data-dependent prior~\cite{ilyas2018prior} along with the transfer-based prior into our algorithm. The data-dependent prior
is proposed to reduce query complexity, which suggests that we can utilize the structure of the inputs to reduce the input-space dimension without sacrificing much estimation accuracy. This idea has also been adopted in several works~\cite{Chen2017ZOO,tu2018autozoom,guo2018low,brunner2018guessing}. We observe that many works restrict the adversarial perturbations to lie in a linear subspace of the input space, which allows the application of our theoretical framework.

Consider the RGF estimator in Eq.~\eqref{eq:estimate}. 
To leverage the data-dependent prior, suppose $u_i=\mathbf{V}\xi_i$, where $\mathbf{V}=[v_1,v_2,...,v_d]$ is a $D\times d$ matrix ($d<D$), $\{v_j\}_{j=1}^d$ is an orthonormal basis in the $d$-dimensional subspace of the input space, and $\xi_i$ is a random vector sampled from the $d$-dimensional unit hypersphere. When $\xi_i$ is sampled from the uniform distribution, $\mathbf{C}=\frac{1}{d}\sum_{i=1}^d v_i v_i^\top$.

Specifically, we focus on the data-dependent prior in~\cite{ilyas2018prior}. In this method, the random vector $\xi_i$ drawn in $\mathbb{R}^d$ is up-sampled to $u_i$ in $\mathbb{R}^D$ by the nearest neighbor algorithm, where $d<D$. The orthonormal basis $\{v_j\}_{j=1}^d$ can be obtained by first up-sampling the standard basis in $\mathbb{R}^d$ with the same method and then applying normalization.

Now we consider incorporating the data-dependent prior into our algorithm.
Similar to Eq.~\eqref{eq:mix_C}, we let one eigenvector of $\mathbf{C}$ be $v$ to exploit the transfer-based prior, and the others are given by the orthonormal basis in the subspace to exploit the data-dependent prior, as
\begin{equation}
    \mathbf{C}=\lambda v v^\top + \frac{1-\lambda}{d}\sum_{i=1}^d v_i v_i^\top.
    \label{eq:mix_dp_C}
\end{equation}
By plugging Eq.~\eqref{eq:mix_dp_C} into Eq.~\eqref{eq:biased_obj}, we can also obtain the optimal $\lambda$ (proof in Appendix \ref{sec:proof_eq-lambda-2}) as
\begin{align}
\small
    \lambda^* =
    \begin{cases}
        \hfil 0 & \text{if} \; \alpha^2 \leq \dfrac{A^2}{d+2q-2} \\
        \dfrac{A^2(A^2-\alpha^2(d+2q-2))}{A^4+\alpha^4d^2-2A^2\alpha^2(q+dq-1)} & \text{if} \; \dfrac{A^2}{d+2q-2}<\alpha^2 < \dfrac{A^2(2q-1)}{d}\\
        \hfil 1 & \text{if} \; \alpha^2\geq \dfrac{A^2(2q-1)}{d}
    \end{cases},
    \label{eq:lambda-2}
\end{align}
where $A^2=\sum_{i=1}^d (v_i^\top\overline{\nabla f(x)})^2$. Note that $A$ should also be estimated. We use a similar method to the one for estimating $\alpha$ in Sec.~\ref{sec:alpha}, which is provided in Appendix \ref{sec:estimation-A}.

The remaining problem is to construct a random vector $u_i$ satisfying $\E[u_i u_i^\top]=\mathbf{C}$, with $\mathbf{C}$ defined in Eq.~\eqref{eq:mix_dp_C}. In general, it is difficult since $v$ is not orthogonal to the subspace.
To address this issue, we sample $u_i$ in a way that $\E[u_i u_i^\top]$ is a good approximation of $\mathbf{C}$ (explanation in Appendix \ref{sec:proof_eq-sampling-dp}), which is similar to Eq.~\eqref{eq:sample_rv} as
\begin{equation}
    u_i=\sqrt{\lambda}\cdot v+\sqrt{1-\lambda}\cdot\overline{(\mathbf{I}-vv^\top)\mathbf{V}\xi_i},
    \label{eq:sampling-dp}
\end{equation}
where $\xi_i$ is sampled uniformly from the $d$-dimensional unit hypersphere. 

Our algorithm with the data-dependent prior is similar to Algorithm~\ref{alg:biased}.
We first estimate $\alpha$ and $A$, and then calculate $\lambda^*$ by Eq.~\eqref{eq:lambda-2}.
If $\lambda^*=1$, we use the transfer gradient $v$ as the estimate.
If not, we sample $q$ random vectors by Eq.~\eqref{eq:sampling-dp} and obtain the gradient estimation by Eq.~\eqref{eq:estimate}.

\begin{figure}[t]
\centering
\includegraphics[width=1.0\linewidth]{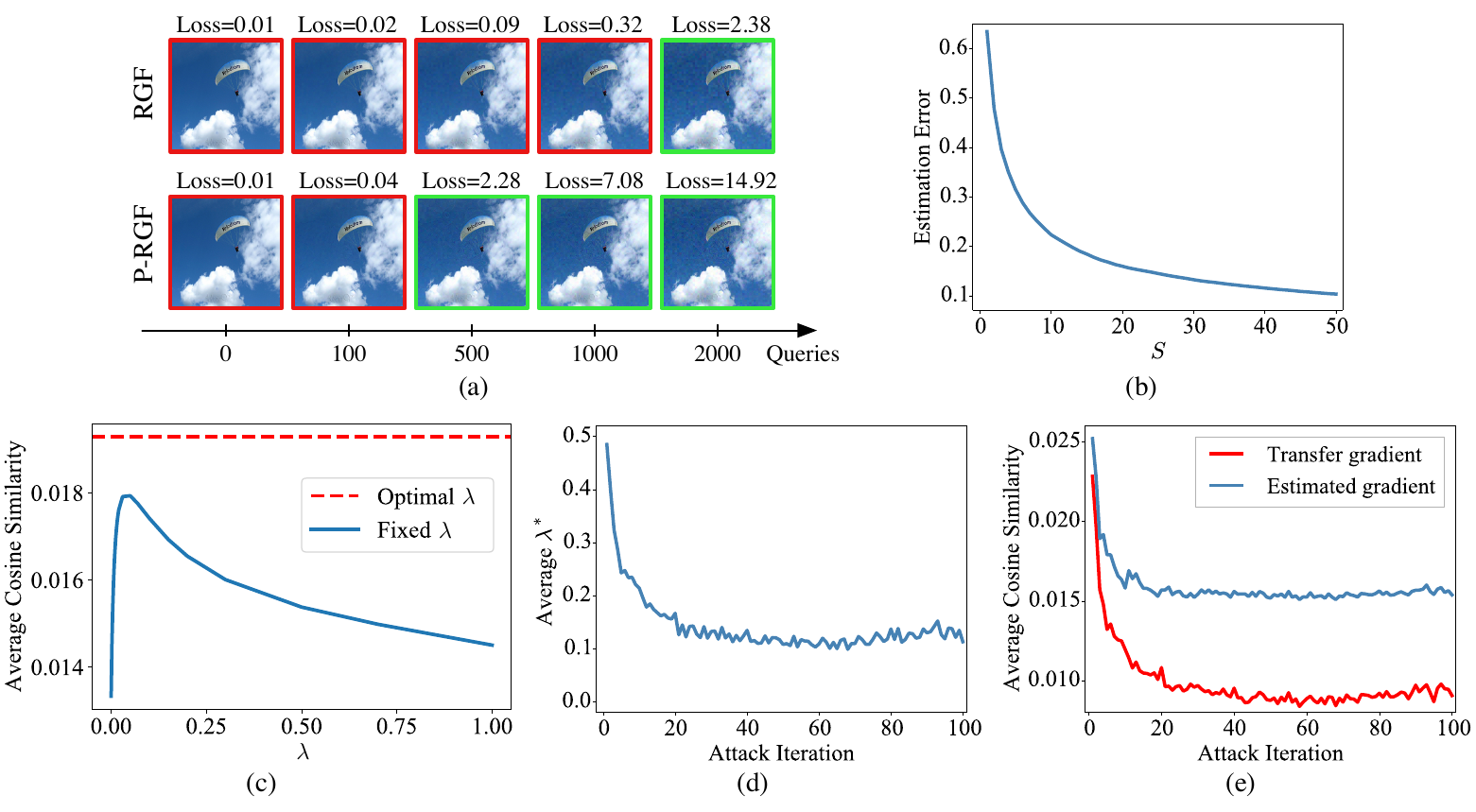}
\caption{(a) The crafted adversarial examples for the Inception-v3~\cite{szegedy2016rethinking} model by RGF and our P-RGF w.r.t. number of queries. We show the cross-entropy loss of each image. The images in the green boxes are successfully misclassified, while those in the red boxes are not. (b) The estimation error of gradient norm with different $S$. (c) The average cosine similarity between the estimated gradient and the true gradient. The estimate is given by our method with fixed $\lambda$ and optimal $\lambda$, respectively. (d) The average $\lambda^*$ across attack iterations. (e) The average cosine similarity between the transfer and the true gradients, and that between the estimated and the true gradients, across attack iterations.}
\label{fig:demo}
\end{figure}

\section{Experiments}
In this section, we present the experimental results to demonstrate the effectiveness of the proposed method on attacking black-box classifiers.\footnote{Our code is available at: \url{https://github.com/thu-ml/Prior-Guided-RGF}.}
We perform untargeted attacks under both the $\ell_2$ and $\ell_\infty$ norms on the ImageNet dataset~\cite{russakovsky2015imagenet}. We choose $1$,$000$ images randomly from the validation set for evaluation.
Due to the space limitation, we leave the results based on the $\ell_\infty$ norm in Appendix \ref{sec:additional-exps}.
The results for both norms are consistent.
For all experiments, we use the ResNet-152 model~\cite{He2016} as the surrogate model to generate the transfer gradient.
We apply the PGD algorithm~\cite{madry2017towards} to generate adversarial examples with the estimated gradient given by each method. We set the perturbation size as $\epsilon=\sqrt{0.001\cdot D}$ and the learning rate as $\eta=2$ in PGD under the $\ell_2$ norm, with images in $[0,1]$.

\begin{table}
  \caption{The experimental results of black-box attacks against Inception-v3, VGG-16, and ResNet-50 under the $\ell_2$ norm. We report the attack success rate (ASR) and the average number of queries (AVG. Q) needed to generate an adversarial example over successful attacks.}
  \label{tab:final-results}
  \centering
  \small
  
  \begin{tabular}{l|cc|cc|cc}
    \hline
    \multirow{2}{*}{Methods} & \multicolumn{2}{c|}{Inception-v3} & \multicolumn{2}{c|}{VGG-16} & \multicolumn{2}{c}{ResNet-50}\\
    \cline{2-7}
    & ASR & AVG. Q & ASR & AVG. Q & ASR & AVG. Q \\
    \hline
    NES~\cite{ilyas2018black} & 95.5\% & 1718 & 98.7\% & 1081 & 98.4\% & 969 \\
    Bandits\textsubscript{T}~\cite{ilyas2018prior} & 92.4\% & 1560 & 94.0\% & 584 & 96.2\% & 1076 \\
    Bandits\textsubscript{TD}~\cite{ilyas2018prior} & 97.2\% & 874 & 94.9\% & 278 & 96.8\% & 512 \\
    AutoZoom~\cite{tu2018autozoom} & 85.4\% & 2443 & 96.2\% & 1589 & 94.8\% & 2065 \\
    \hline
    RGF & 97.7\% & 1309 & \bf99.8\% & 749 & \bf99.6\% & 673 \\
    P-RGF ($\lambda=0.5$) & 96.5\% & 1119 & 97.8\% & 710 & 98.7\% & 635 \\
    P-RGF ($\lambda=0.05$) & 97.8\% & 1021 & 99.7\% & 624 & 99.3\% & 511 \\
    P-RGF ($\lambda^*$) & \bf98.1\% & \bf745 & 99.6\% & \bf331 & \bf99.6\% & \bf265 \\
    \hline
    RGF\textsubscript{D} & \bf99.1\% & 910 & \bf100.0\% & 372 & 99.7\% & 429 \\
    P-RGF\textsubscript{D} ($\lambda=0.5$) & 98.2\% & 1047 & 99.7\% & 634 & 99.5\% & 552 \\
    P-RGF\textsubscript{D} ($\lambda=0.05$) & \bf99.1\% & 754 & 99.9\% & 359 & \bf99.8\% & 379 \\
    P-RGF\textsubscript{D} ($\lambda^*$) & \bf99.1\% & \bf649 & 99.8\% & \bf250 & 99.6\% & \bf232 \\
    \hline
  \end{tabular}
\end{table}

\subsection{Performance of gradient estimation}
\label{sec:estimation}

In this section, we conduct several experiments to show the performance of gradient estimation. All experiments in this section are performed on the Inception-v3~\cite{szegedy2016rethinking} model.

First, we show the performance of gradient norm estimation in Sec.~\ref{sec:alpha}. The gradient norm (or cosine similarity) is easier to estimate than the true gradient since it's a scalar value. Fig.~\ref{fig:demo}(b) shows the estimation error of the gradient norm, defined as the (normalized) RMSE---$\sqrt{\mathbb{E}\big(\frac{\widehat{\|\nabla f(x)\|_2} - \|\nabla f(x)\|_2 }{\|\nabla f(x)\|_2}\big)^2}$, w.r.t. the number of queries $S$, where $\|\nabla f(x)\|_2$ is the true norm and $\widehat{\|\nabla f(x)\|_2}$ is the estimated one. We choose $S=10$ in all experiments to reduce the number of queries while the estimation error is acceptable. We also estimate the gradient norm every $10$ attack iterations in all experiments to reduce the number of queries, since usually its value is relatively stable in the optimization process.

Second, we verify the effectiveness of the derived optimal $\lambda$ (i.e., $\lambda^*$) in Eq.~\eqref{eq:lambda-1} for gradient estimation, compared with any fixed $\lambda\in[0,1]$. We perform attacks against Inception-v3 using P-RGF with $\lambda^*$, and calculate the cosine similarity between the estimated gradient and the true gradient. We calculate $\lambda^*$ using the estimated $\alpha$ in Sec.~\ref{sec:alpha} instead of its true value. Meanwhile, along the PGD updates, we also use fixed $\lambda$ to get gradient estimates, and calculate the corresponding cosine similarities. Note that $\lambda^*$ does not correspond to any fixed value, since it varies during iterations. 

The average cosine similarities of different values of $\lambda$ are shown in Fig.~\ref{fig:demo}(c).
It can be observed that when a suitable value of $\lambda$ is chosen, the proposed P-RGF provides a better gradient estimate than both the original RGF method with uniform distribution (when $\lambda=\frac{1}{D}\approx 0$) and the transfer gradient (when $\lambda=1$). Adopting $\lambda^*$ brings further improvement upon any fixed $\lambda$, demonstrating the applicability of our theoretical framework. 

Finally, we show the average $\lambda^*$ over all images w.r.t. attack iterations in Fig.~\ref{fig:demo}(d). It shows that $\lambda^*$ decreases along with the iterations. Fig.~\ref{fig:demo}(e) shows the average cosine similarity between the transfer and the true gradients, and that between the estimated and the true gradients, across iterations. The results show that the transfer gradient is useful at beginning, and becomes less useful along with the iterations. However, the estimated gradient can remain higher cosine similarity with the true gradient, which facilitates the adversarial attacks consequently. The results also prove that we need to use the adaptive $\lambda^*$ in different attack iterations.

\begin{figure}[t]
\centering
\includegraphics[width=1\linewidth]{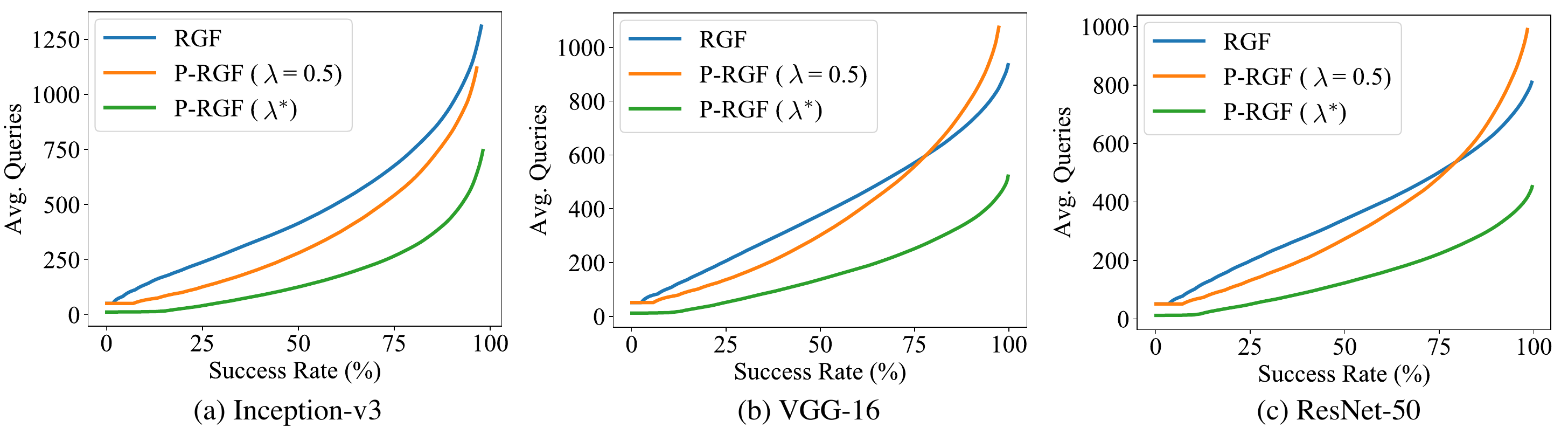}
\caption{We show the average number of queries per successful image at any desired success rate.}
\label{fig:curve}
\end{figure}

\subsection{Results of black-box attacks on normal models}
\label{sec:results}
In this section, we perform attacks against three normally trained models, which are Inception-v3~\cite{szegedy2016rethinking}, VGG-16~\cite{simonyan2014very}, and ResNet-50~\cite{He2015}. 
We compare the proposed prior-guided random gradient-free (P-RGF) method with two baseline methods. The first is the original RGF method with uniform sampling. The second is the P-RGF method with a fixed $\lambda$, which is set to $0.5$ or $0.05$.
In these methods, we set the number of queries as $q=50$, and the sampling variance as $\sigma=0.0001\cdot\sqrt{D}$.
We also incorporate the data-dependent prior into these three methods for comparison (which are denoted by adding a subscript ``D''). We set the dimension of the subspace as $d=50\times50\times3$. Besides, our method is compared with the state-of-the-art attack methods, including the natural evolution strategy (NES)~\cite{ilyas2018black}, bandit optimization methods (Bandits\textsubscript{T} and Bandits\textsubscript{TD})~\cite{ilyas2018prior}, and AutoZoom~\cite{tu2018autozoom}.
For all methods, we restrict the maximum number of queries for each image to be $10$,$000$.
We report a successful attack if a method can generate an adversarial example within $10$,$000$ queries and the size of perturbation is smaller than the budget (i.e., $\epsilon=\sqrt{0.001\cdot D}$).

We show the success rate of black-box attacks and the average number of queries needed to generate an adversarial example over the successful attacks in Table~\ref{tab:final-results}. It can be seen that our method generally leads to higher attack success rates and requires much fewer queries than other methods. Using a fixed $\lambda$ cannot give a satisfactory result, which demonstrates the necessity of using the optimal $\lambda$ in our method. The results also show that the data-dependent prior is orthogonal to the proposed transfer-based prior, since integrating the data-dependent prior leads to better results.
We show an example of attacks in Fig.~\ref{fig:demo}(a).
Fig.~\ref{fig:curve} shows the average number of queries over successful images by reaching a desired success rate. Our method is much more query-efficient than baseline methods.

\begin{table}[t]
  \caption{The experimental results of black-box attacks against JPEG compression~\cite{Guo2017Countering}, randomization~\cite{Xie2018Mitigating}, and guided denoiser~\cite{Liao2017Defense} under the $\ell_2$ norm. We report the attack success rate (ASR) and the average number of queries (AVG. Q) needed to generate an adversarial example over successful attacks.}
  \label{tab:defense-results}
  \centering
  \small
  
  \begin{tabular}{l|cc|cc|cc}
    \hline
    \multirow{2}{*}{Methods} & \multicolumn{2}{c|}{JPEG Compression~\cite{Guo2017Countering}} & \multicolumn{2}{c|}{Randomization~\cite{Xie2018Mitigating}} & \multicolumn{2}{c}{ Guided Denoiser~\cite{Liao2017Defense}}\\
    \cline{2-7}
    & ASR & AVG. Q & ASR & AVG. Q & ASR & AVG. Q \\
    \hline
    NES~\cite{ilyas2018black} & 47.3\% & 3114 & 23.2\% & 3632 & 48.0\% & 3633 \\
    SPSA~\cite{uesato2018adversarial} & 40.0\% & 2744 & 9.6\% & 3256 & 46.0\% & 3526 \\
    \hline
    RGF & 41.5\% & 3126 & 19.5\% & 3259 & 50.3\% & 3569 \\
    P-RGF & \bf61.4\% & \bf2419 & \bf60.4\% & \bf2153 & \bf51.4\% & \bf2858 \\
    \hline
    RGF\textsubscript{D} & 70.4\% & 2828 & 54.9\% & 2819 & 83.7\% & 2230 \\
    P-RGF\textsubscript{D} & \bf81.1\% & \bf2120 & \bf82.3\% & \bf1816 & \bf89.6\% & \bf1784 \\
    \hline
  \end{tabular}
\end{table}

\subsection{Results of black-box attacks on defensive models}
We further validate the effectiveness of the proposed method on attacking several defensive models, including JPEG compression~\cite{Guo2017Countering}, randomization~\cite{Xie2018Mitigating}, and guided denoiser~\cite{Liao2017Defense}.
We utilize the Inception-v3 model as the backbone classifier for the JPEG compression and randomization defenses.
We compare P-RGF with RGF, NES~\cite{ilyas2018black}, and SPSA~\cite{uesato2018adversarial}.
The experimental settings are the same with those of attacking the normal models in Sec.~\ref{sec:results}.
In our method, we use a smoothed version of the transfer gradient~\cite{dong2019evading} as the transfer-based prior for black-box attacks, since the smoothed transfer gradient is better to defeat defensive models.
The results in Table~\ref{tab:defense-results} also demonstrate the superiority of our method for attacking the defensive models. Our method leads to much higher attack success rates than other methods ($20\%\sim40\%$ improvements in many cases), and also reduces the query complexity.

\section{Conclusion}

In this paper, we proposed a prior-guided random gradient-free (P-RGF) method to utilize the transfer-based prior for improving black-box adversarial attacks.
Our method appropriately integrated the transfer gradient of a surrogate white-box model by the derived optimal coefficient.
The experimental results consistently demonstrate the effectiveness of our method, which requires much fewer queries to attack black-box models with higher success rates.

\section*{Acknowledgements}
This work was supported by the National Key Research and Development Program of China (No. 2017YFA0700904), NSFC Projects (Nos. 61620106010, 61621136008, 61571261), Beijing NSF Project (No. L172037), Beijing Academy of Artificial Intelligence (BAAI), Tiangong Institute for Intelligent Computing, the JP Morgan Faculty Research Program and the NVIDIA NVAIL Program with GPU/DGX Acceleration.

\clearpage
\small
\bibliographystyle{plainnat}
\bibliography{reference}

\clearpage
\normalsize
\setcounter{table}{2}
\numberwithin{equation}{section}

\appendix

\section{Proofs}
We provide the proofs in this section.
\subsection{Proof of Theorem \protect\ref{the:1}}
\label{sec:proof_the-1}
\textbf{Theorem 1.} \textit{If $f$ is differentiable at $x$, the loss of the RGF estimator $\hat{g}$ is
\begin{equation*}
    \lim_{\sigma\to 0}L(\hat{g})=\|\nabla f(x)\|_2^2-\frac{\big(\nabla f(x)^\top\mathbf{C}\nabla f(x)\big)^2}{(1-\frac{1}{q})\nabla f(x)^\top\mathbf{C}^2\nabla f(x)+\frac{1}{q}\nabla f(x)^\top\mathbf{C}\nabla f(x)},
\end{equation*}
where $\sigma$ is the sampling variance, $\mathbf{C}=\E[u_iu_i^\top]$ with $u_i$ being the random vector, $\|u_i\|_2=1$, and $q$ is the number of random vectors as in Eq.~\eqref{eq:estimate}.}
\begin{remark}
Rigorously speaking, we assume $\nabla f(x)^\top \mathbf{C} \nabla f(x)\neq 0$ in the statement of the theorem (and also in the proof), since when $\nabla f(x)^\top \mathbf{C} \nabla f(x)\neq 0$, both the numerator and the denominator of the fraction above are zero. When $\nabla f(x)^\top \mathbf{C} \nabla f(x)=0$, $u_i^\top \nabla f(x)=0$ holds almost surely, which implies that $L(\hat{g})=\|\nabla f(x)\|^2$ regardless of the value of $\sigma$. In fact, this case will not happen almost surely. In the setting of black-box attacks, we cannot even design a $\mathbf{C}$ with trace 1 such that $\nabla f(x)^\top \mathbf{C} \nabla f(x)= 0$ since $\nabla f(x)$ is unknown.
\end{remark}
\begin{proof}
First, we derive $L(\hat{g})$ based on the assumption that the single estimate $\hat{g}_i$ in Eq.~\eqref{eq:estimate} is equal to $u_i^\top\nabla f(x) \cdot u_i$, which will hold when $f$ is locally linear.
\begin{lemma}
Assume that the single estimate $\hat{g}_i$ in Eq.~\eqref{eq:estimate} is equal to $u_i^\top\nabla f(x) \cdot u_i$. We have
\begin{equation}
    L(\hat{g})=\|\nabla f(x)\|_2^2-\frac{(\nabla f(x)^\top\mathbf{C}\nabla f(x))^2}{(1-\frac{1}{q})\nabla f(x)^\top\mathbf{C}^2\nabla f(x)+\frac{1}{q}\nabla f(x)^\top\mathbf{C}\nabla f(x)}.
    \label{eq:app_biased_obj}
\end{equation}
\end{lemma}
\begin{proof}
First, we have \[\E\|\nabla f(x)-b\hat{g}\|_2^2=\|\nabla f(x)\|_2^2-2b\nabla f(x)^\top\E[\hat{g}]+b^2\E\|\hat{g}\|_2^2.\]
We have $\nabla f(x)^\top \E[\hat{g}]=\nabla f(x)^\top \E[\hat{g}_i]=\E[\nabla f(x)^\top u_i u_i^\top \nabla f(x)]=\E[(\nabla f(x)^\top u_i)^2]\geq 0$. Hence
\begin{equation}
    L(\hat{g})=\min_{b\geq0} \E\|\nabla f(x)-b\hat{g}\|_2^2 = \min_b \E\|\nabla f(x)-b\hat{g}\|_2^2 = \|\nabla f(x)\|_2^2-\frac{(\nabla f(x)^\top \E[\hat{g}])^2}{\E\|\hat{g}\|_2^2}.
    \label{eq:ghat_abstract}
\end{equation}
Since $\hat{g}_i=u_i^\top \nabla f(x) \cdot u_i$, and $u_i^\top u_i\equiv 1$, we have
\begin{align*}
    \E[\hat{g}_i] &= \mathbf{C}\nabla f(x), \\
    \E\|\hat{g}_i\|_2^2 &= \E[\hat{g}_i^\top \hat{g}_i] \\ &= \E[\nabla f(x)^\top u_i u_i^\top u_i u_i^\top \nabla f(x)] \\ &= \nabla f(x)^\top\E[u_i (u_i^\top u_i) u_i^\top]\nabla f(x) \\&= \nabla f(x)^\top\E[u_i  u_i^\top]\nabla f(x) \\&= \nabla f(x)^\top\mathbf{C}\nabla f(x).
\end{align*}
Given $\E[\hat{g}_i]$ and $\E\|\hat{g}_i\|^2$, the corresponding moments of $\hat{g}$ can be computed as
\begin{align}
    \E[\hat{g}]&=\E[\hat{g}_i] \label{eq:ghat_as_gi}\\
    &=\mathbf{C}\nabla f(x), \nonumber\\
    \E\|\hat{g}\|_2^2&=\E\|\hat{g}-\E[\hat{g}]\|_2^2+\|\E[\hat{g}]\|_2^2 \nonumber\\
    &=\frac{1}{q}\E\|\hat{g}_i-\E[\hat{g}_i]\|_2^2+\|\E[\hat{g}_i]\|_2^2 \nonumber\\
    &=\frac{1}{q}\E\|\hat{g}_i\|_2^2+(1-\frac{1}{q})\|\E[\hat{g}_i]\|_2^2 \label{eq:ghat_as_gi_2}\\
    &=(1-\frac{1}{q})\nabla f(x)^\top\mathbf{C}^2\nabla f(x)+\frac{1}{q}\nabla f(x)^\top\mathbf{C}\nabla f(x). \nonumber
\end{align}
Plug them into Eq.~\eqref{eq:ghat_abstract} and we complete the proof.
\end{proof}
Next, we prove that if $f$ is not locally linear, as long as it is differentiable at $x$, then by picking a sufficient small $\sigma$, the loss tends to be that of the local linear approximation.
\begin{lemma}\label{thm:differentiable}
If $f$ is differentiable at $x$, letting $L_0$ denote the right-hand side of Eq.~\eqref{eq:app_biased_obj}, then we have
\begin{equation*}
    \lim_{\sigma\to 0}L(\hat{g})=L_0.
\end{equation*}
\end{lemma}

\begin{proof}

Let $\hat{g}_i'=u_i^\top\nabla f(x)\cdot u_i$, $\hat{g}'=\frac{1}{q}\sum_{i=1}^q \hat{g}_i'$. Then $L_0=L(\hat{g}')$. By Eq.~\eqref{eq:ghat_abstract}, Eq.~\eqref{eq:ghat_as_gi} and Eq.~\eqref{eq:ghat_as_gi_2}, it suffices to prove $\lim_{\sigma\to 0}\E[\hat{g}_i]=\E[\hat{g}_i']$ and $\lim_{\sigma\to 0}\E\|\hat{g}_i\|_2^2=\E\|\hat{g}_i'\|_2^2$.

For clarity, we redefine the notation: We omit the subscript $i$, make the dependence of $\hat{g}_i$ on $\sigma$ explicit (let $\hat{g}_\sigma$ denote $\hat{g}_i$), and let $\hat{g}_0$ denote $\hat{g}_i'$. Then we omit the hat in $\hat{g}$. That is, let  $g_0\triangleq u^\top \nabla f(x)\cdot u$ and $g_\sigma\triangleq \frac{f(x+\sigma u)-f(x)}{\sigma}\cdot u$, where $u$ is sampled uniformly from the unit hypersphere. Then we want to prove $\lim_{\sigma\to 0}\E[g_\sigma]=\E[g_0]$ and $\lim_{\sigma\to 0}\E\|g_\sigma\|_2^2=\E\|g_0\|_2^2$.

Since $f$ is differentiable at $x$, we have
\begin{equation}
    \lim_{\sigma\to 0}\sup_{\|u\|_2=1}\Big|\frac{f(x+\sigma u)-f(x)}{\sigma}- u^\top\nabla f(x)\Big|=0.
    \label{eq:differentiable}
\end{equation}
Since $\|u\|_2\equiv 1$, we have
\begin{align*}
    \lim_{\sigma\to 0}\E\|g_\sigma-g_0\|_2&\leq \lim_{\sigma\to 0}\sup_{\|u\|_2=1}\Big|\frac{f(x+\sigma u)-f(x)}{\sigma}-u^\top\nabla f(x)\Big|=0,\\
    \lim_{\sigma\to 0}\E\|g_\sigma-g_0\|_2^2&\leq \lim_{\sigma\to 0}\sup_{\|u\|_2=1}\big|\frac{f(x+\sigma u)-f(x)}{\sigma}-u^\top\nabla f(x)\big|^2=0.
\end{align*}
Applying Jensen's inequality to convex function $\|\cdot\|_2$, we have $\|\E[g_\sigma]-\E[g_0]\|_2\leq \E\|g_\sigma-g_0\|_2$. Since $\lim_{\sigma\to 0}\E\|g_\sigma-g_0\|_2=0$, we have $\lim_{\sigma\to 0}\E[g_\sigma]=\E[g_0]$.

Since $\big|\|g_\sigma\|_2-\|g_0\|_2\big|\leq\|g_\sigma-g_0\|_2$, $\lim_{\sigma\to 0}\E\|g_\sigma-g_0\|_2=0$ and $\lim_{\sigma\to 0}\E\|g_\sigma-g_0\|_2^2=0$, we have $\lim_{\sigma\to 0}\E\big|\|g_\sigma\|_2-\|g_0\|_2\big|=0$ and $\lim_{\sigma\to 0}\E(\|g_\sigma\|_2-\|g_0\|_2)^2=0$. Also, we have $\|g_0\|_2\leq \|\nabla f(x)\|_2$. Hence, we have
\begin{align*}
    \lim_{\sigma\to 0}\big|\E\|g_\sigma\|_2^2-\E\|g_0\|_2^2\big|
    &\leq \lim_{\sigma\to 0}\E\big|\|g_\sigma\|_2^2-\|g_0\|_2^2\big| \\
    &=\lim_{\sigma\to 0}\E\Big[\big|\|g_\sigma\|_2-\|g_0\|_2\big|\big(\|g_\sigma\|_2+\|g_0\|_2\big)\Big] \\
    &\leq \lim_{\sigma\to 0}\E\Big[\big(\|g_\sigma\|_2-\|g_0\|_2\big)^2+2\|g_0\|_2\big|\|g_\sigma\|_2-\|g_0\|_2\big|\Big] \\
    &\leq \lim_{\sigma\to 0}\E\Big[\big(\|g_\sigma\|_2-\|g_0\|_2\big)^2+2\|\nabla f(x)\|_2\big|\|g_\sigma\|_2-\|g_0\|_2\big|\Big] \\
    &=0.
\end{align*}
The proof is complete.
\end{proof}
By combining the two lemmas above, our proof for the theorem is complete.
\end{proof}

\subsection{Proof of Eq.~(\protect\ref{eq:sample_rv})}
\label{sec:proof_eq-sample-rv}
Suppose $v$ is a fixed random vector and $\|v\|_2=1$. Let the $D$-dimensional random vector $u$ be
\begin{equation*}
    u=\sqrt{\lambda}\cdot v+\sqrt{1-\lambda}\cdot\overline{(\mathbf{I}-vv^\top)\xi},
\end{equation*}
where $\xi$ is sampled uniformly from the unit hypersphere. We want to prove that
\begin{equation*}
\E[u u^\top]=\lambda v v^\top + \frac{1-\lambda}{D-1}(\mathbf{I}-v v^\top).
\end{equation*}
\begin{proof}
Let $r\triangleq \overline{(\mathbf{I}-vv^\top)\xi}$. We choose an orthonormal basis $\{v_1,...,v_D\}$ of $\mathbb{R}^D$ such that $v_1=v$. Then $\xi$ can be written as $\xi=\sum_{i=1}^D a_i v_i$, where $a=(a_1,...,a_D)^\top$ is sampled uniformly from the unit hypersphere. Hence $(\mathbf{I}-vv^\top)\xi=\sum_{i=2}^D a_i v_i$, and $r=\frac{\sum_{i=2}^D a_i v_i}{\sqrt{\sum_{i=2}^D a_i^2}}$. Let $b_i=\frac{a_i}{\sqrt{\sum_{i=2}^D a_i^2}}$ for $i=2,3,...,D$, then $b=(b_2,b_3,...,b_D)^\top$ is sampled uniformly from the $(D-1)$-dimensional unit hypersphere, and $r=\sum_{i=2}^D b_i v_i$. Hence $\E[r]=0$. To compute $\E[r r^\top]$, we need a lemma first.
\begin{lemma}
Suppose $d$ is a positive integer, $u=\sum_{i=1}^d a_i v_i$ where $a=(a_1,...,a_d)^\top$ is sampled uniformly from the $d$-dimensional unit hypersphere, then $\E[uu^\top]=\frac{1}{d}\sum_{i=1}^d v_i v_i^\top$.
\label{lem:covariance}
\end{lemma}
\begin{proof}
$\E[uu^\top] = \E[(\sum_{i=1}^d a_i v_i)(\sum_{j=1}^d a_j v_j^\top)] = \sum_{i=1}^d \sum_{j=1}^d v_i v_j^\top\E[a_i a_j]$. By symmetry, we have $\E[a_i a_j]=0$ when $i\neq j$, and $\E[a_i^2]=\E[a_j^2]$ for any $i, j$. Since $\sum_{i=1}^d a_i^2=1$, we have $\E[a_i^2]=\frac{1}{d}$ for any $i$. Hence $\E[uu^\top]=\frac{1}{d}\sum_{i=1}^d v_i v_i^\top$.
\end{proof}
Using the lemma, we have $\E[r r^\top]=\frac{1}{D-1}\sum_{i=2}^D v_i v_i^\top=\frac{1}{D-1}(\mathbf{I}-vv^\top)$. Since $\E[r]=0$, we have $\E[vr^\top]=\E[rv^\top]=0$. Hence, we have
\begin{align*}
    \E[uu^\top]&=\E[(\sqrt{\lambda}\cdot v+\sqrt{1-\lambda}\cdot r)(\sqrt{\lambda}\cdot v+\sqrt{1-\lambda}\cdot r)^\top] \\
    &=\lambda vv^\top+(1-\lambda)\E[rr^\top] \\
    &=\lambda vv^\top+\frac{1-\lambda}{D-1}(\mathbf{I}-vv^\top).
\end{align*}
The proof is complete.
\end{proof}
\begin{remark}
The construction of the random vector $u$ such that $\E[u u^\top]=\lambda v v^\top + \frac{1-\lambda}{D-1}(\mathbf{I}-v v^\top)$ is not unique. One can choose a different kind of distribution or simply take the negative of $u$ while remaining $\E[u u^\top]$ invariant.
\end{remark}

\subsection{Proof of Eq.~(\protect\ref{eq:lambda-1})}
\label{sec:proof_eq-lambda}
Let $\alpha=v^\top\overline{\nabla f(x)}$. Suppose $D\geq 2$, $q\geq 1$. After plugging Eq.~\eqref{eq:mix_C} into Eq.~\eqref{eq:biased_obj}, the optimal $\lambda$ is given by
\begin{align}
    \lambda^* =
    \begin{cases}
        \hfil 0 & \text{if} \; \alpha^2\leq\dfrac{1}{D+2q-2}  \\
        \dfrac{(1-\alpha^2)(\alpha^2(D+2q-2)-1)}{2\alpha^2 Dq-\alpha^4 D(D+2q-2)-1} & \text{if} \; \dfrac{1}{D+2q-2} < \alpha^2 < \dfrac{2q-1}{D+2q-2} \\
        \hfil 1 & \text{if} \; \alpha^2 \geq \dfrac{2q-1}{D+2q-2}
    \end{cases}.
    \label{eq:app_lambda-1}
\end{align}

\begin{proof}
After plugging Eq.~\eqref{eq:mix_C} into Eq.~\eqref{eq:biased_obj}, we have
\begin{align*}
    L(\lambda)&=\|\nabla f(x)\|_2^2 \Big(1-\frac{(\lambda \alpha^2+\frac{1-\lambda}{D-1}(1-\alpha^2))^2}{(1-\frac{1}{q})(\lambda^2 \alpha^2+(\frac{1-\lambda}{D-1})^2(1-\alpha^2))+\frac{1}{q}(\lambda \alpha^2+\frac{1-\lambda}{D-1}(1-\alpha^2))}\Big).
\end{align*}
To minimize $L(\lambda)$, we should maximize
\begin{equation}
    F(\lambda)=\frac{(\lambda \alpha^2+\frac{1-\lambda}{D-1}(1-\alpha^2))^2}{(1-\frac{1}{q})(\lambda^2 \alpha^2+(\frac{1-\lambda}{D-1})^2(1-\alpha^2))+\frac{1}{q}(\lambda \alpha^2+\frac{1-\lambda}{D-1}(1-\alpha^2))}.
\label{eq:opt_F}
\end{equation}
Note that $F(\lambda)$ is a quadratic rational function w.r.t. $\lambda$.

Since we optimize $\lambda$ in a closed interval $[0, 1]$, checking $\lambda=0$, $\lambda=1$ and the stationary points (such that $F'(\lambda)=0$) would suffice. By solving $F'(\lambda)=0$, we have at most two solutions:
\begin{align}
    \lambda_1&=\frac{(1-\alpha^2)(\alpha^2(D+2q-2)-1)}{2\alpha^2 Dq-\alpha^4 D(D+2q-2)-1}, \label{eq:sol_lambda_1} \\
    \lambda_2&=\frac{1-\alpha^2}{1-\alpha^2 D} \nonumber,
\end{align}
where $\lambda_1$ or $\lambda_2$ is the solution if and only if the denominator is not 0. Given $\alpha^2\leq 1$ and $D\geq 2$, $\lambda_2\notin (0,1)$, so we only need to consider $\lambda_1$.

First, we figure out when $\lambda_1\in (0,1)$. We can verify that $\lambda_1=1$ when $\alpha^2=0$ and $\lambda_1=0$ when $\alpha^2=1$. Suppose $\alpha^2\in (0,1)$. Let $J$ denote the numerator in Eq.~\eqref{eq:sol_lambda_1} and $K$ denote the denominator. We have that when $\alpha^2> \frac{1}{D+2q-2}$, $J> 0$; else, $J\leq 0$. We also have that when $\alpha^2< \frac{2q-1}{D+2q-2}$, $J < K$; else, $J\geq K$. Note that $J/K\in (0,1)$ if and only if $0< J < K$ or $0> J> K$. Hence, $\lambda_1\in (0,1)$ if and only if $\frac{1}{D+2q-2} < \alpha^2 < \frac{2q-1}{D+2q-2}$.

Case 1: $\lambda_1\notin (0,1)$. Then it suffices to compare $F(0)$ with $F(1)$. We have
\begin{equation*}
    F(0)=\frac{(1-\alpha^2)q}{D+q-2}, F(1)=\alpha^2.
\end{equation*}
Hence, $F(0)\geq F(1)$ if and only if $\alpha^2\leq \frac{q}{D+2q-2}$. It means that if $\alpha^2\geq\frac{2q-1}{D+2q-2}$, then $\lambda^*=1$; if $\alpha^2\leq\frac{1}{D+2q-2}$, then $\lambda^*=0$.

Case 2: $\lambda_1\in (0,1)$. After plugging Eq.~\eqref{eq:sol_lambda_1} into Eq.~\eqref{eq:opt_F}, we have
\begin{align}
    F(\lambda_1)&=\frac{4\alpha^2(1-\alpha^2)(q-1)q}{-1+2\alpha^2(D(2q-1)+2(q-1)^2)-\alpha^4(D+2q-2)^2}. \label{eq:F_lambda_1}
\end{align}
Now we prove that $F(\lambda_1)\geq F(0)$ and $F(\lambda_1)\geq F(1)$. Since when $0< \lambda < 1$, both the numerator and the denominator in Eq.~\eqref{eq:opt_F} is positive, we have $F(\lambda) > 0$, $\forall \lambda\in (0,1)$. Since the numerator in Eq.~\eqref{eq:F_lambda_1} is non-negative and $F(\lambda_1)>0$, we know that the denominator in Eq.~\eqref{eq:F_lambda_1} is positive. Hence, we have
\begin{align*}
    F(\lambda_1)-F(0)&=\frac{q(1-\alpha^2)(\alpha^2(D+2q-2)-1)^2}{(q+D-2)(-1+2\alpha^2(D(2q-1)+2(q-1)^2)-\alpha^4(D+2q-2)^2)}> 0; \\
    F(\lambda_1)-F(1)&=\frac{\alpha^2(\alpha^2(D+2q-2)+1-2q)^2}{-1+2\alpha^2(D(2q-1)+2(q-1)^2)-\alpha^4(D+2q-2)^2}> 0.
\end{align*}
Hence in this case $\lambda^*=\lambda_1$.

The proof is complete.
\end{proof}

\subsection{Monotonicity of $\lambda^*$}
\label{sec:proof_monotonicity-lambda}
We will prove that $\lambda^*$ is a monotonically increasing function of $\alpha^2$, and a monotonically decreasing function of $q$ (when $\alpha^2>\frac{1}{D}$).
\begin{proof}
To find the monotonicity w.r.t. $\alpha^2$, note that $\lambda^*=0$ if $\alpha^2\leq \frac{1}{D+2q-2}$ and $\lambda^*=1$ when $\alpha^2\geq \frac{2q-1}{D+2q-2}$. When $\frac{1}{D+2q-2}<\alpha^2<\frac{2q-1}{D+2q-2}$, we have
\begin{align}
    \lambda^* & = \dfrac{(1-\alpha^2)(\alpha^2(D+2q-2)-1)}{2\alpha^2 Dq-\alpha^4 D(D+2q-2)-1} \nonumber\\
    & = \dfrac{\alpha^4(D+2q-2)-\alpha^2(D+2q-1)+1}{\alpha^4D(D+2q-2)-2\alpha^2Dq+1} \nonumber\\
    & = \frac{1}{D}\Big(1-\frac{(\alpha^2D-1)(D-1)}{\alpha^4D(D+2q-2)-2\alpha^2Dq+1}\Big) \label{eq:lambda-mono}\\
    & = \frac{1}{D} - \frac{D-1}{\alpha^2D(D+2q-2) - (2Dq-D-2q+2) - 2\frac{(D-1)(q-1)}{\alpha^2D-1}}. \nonumber
\end{align}
When $\alpha^2 <\frac{1}{D}$, or when $\alpha^2 >\frac{1}{D}$, a larger $\alpha^2$ leads to larger values of both $\alpha^2D(D+2q-2)$ and $-2\frac{(D-1)(q-1)}{\alpha^2D-1}$, and consequently leads to a larger $\lambda^*$. Meanwhile, by the argument in the proof of Eq.~(12), when $\frac{1}{D+2q-2}<\alpha^2<\frac{2q-1}{D+2q-2}$, the denominator of Eq.~\eqref{eq:sol_lambda_1} is positive, hence $\alpha^4D(D+2q-2)-2\alpha^2Dq+1<0$. By Eq.~\eqref{eq:lambda-mono}, when $\alpha^2 <\frac{1}{D}$, $\lambda^*<\frac{1}{D}$; when $\alpha^2 =\frac{1}{D}$, $\lambda^*=\frac{1}{D}$; when $\alpha^2 >\frac{1}{D}$, $\lambda^*>\frac{1}{D}$. We conclude that $\lambda^*$ is a monotonically increasing function of $\alpha^2$.

To find the monotonicity w.r.t $q$ when $\alpha^2>\frac{1}{D}$, Eq.~\eqref{eq:app_lambda-1} tells us that when $q\leq\frac{\alpha^2(D-2)+1}{2(1-\alpha^2)}$, $\lambda^*=1$; else, $0<\lambda^*<1$. In the latter case, we rewrite Eq.~\eqref{eq:lambda-mono} as
\begin{align*}
    \lambda^* & = \frac{1}{D}\Big(1+\frac{(\alpha^2D-1)(D-1)}{2\alpha^2D(1-\alpha^2)q-\alpha^4D(D-2)-1}\Big).
\end{align*}
We have $(\alpha^2 D-1)(D-1)>0$, and as explained before, the denominator is positive for any $q$ such that $0<\lambda^*<1$. Hence, when $\alpha^2>\frac{1}{D}$, $\lambda^*$ is a monotonically decreasing function of $q$.
\end{proof}

\subsection{Proof of Eq.~(\protect\ref{eq:lambda-2})}
\label{sec:proof_eq-lambda-2}
Let $\alpha=v^\top\overline{\nabla f(x)}$, $A^2=\sum_{i=1}^d (v_i^\top\overline{\nabla f(x)})^2$. Suppose $d\geq 1$, $q\geq 1$. After plugging Eq.~\eqref{eq:mix_dp_C} into Eq.~\eqref{eq:biased_obj}, the optimal $\lambda$ is given by
\begin{align*}
\small
    \lambda^* =
    \begin{cases}
        \hfil 0 & \text{if} \; \alpha^2 < \dfrac{A^2}{d+2q-2} \\
        \dfrac{A^2(A^2-\alpha^2(d+2q-2))}{A^4+\alpha^4d^2-2A^2\alpha^2(q+dq-1)} & \text{if} \; \dfrac{A^2}{d+2q-2}\leq\alpha^2 < \dfrac{A^2(2q-1)}{d}\\
        \hfil 1 & \text{if} \; \alpha^2\geq \dfrac{A^2(2q-1)}{d}
    \end{cases}.
\end{align*}
\begin{proof}
The proof is very similar to that in Sec.~\ref{sec:proof_eq-lambda}. After plugging Eq.~\eqref{eq:mix_dp_C} into Eq.~\eqref{eq:biased_obj}, we have
\begin{align*}
    L(\lambda)&=\|\nabla f(x)\|_2^2 \Big(1-\frac{(\lambda \alpha^2+\frac{1-\lambda}{d}A^2)^2}{(1-\frac{1}{q})(\lambda^2 \alpha^2+(\frac{1-\lambda}{d})^2 A^2)+\frac{1}{q}(\lambda \alpha^2+\frac{1-\lambda}{d}A^2)}\Big).
\end{align*}
To minimize $L(\lambda)$, we should maximize
\begin{equation}
    F(\lambda)=\frac{(\lambda \alpha^2+\frac{1-\lambda}{d}A^2)^2}{(1-\frac{1}{q})(\lambda^2 \alpha^2+(\frac{1-\lambda}{d})^2 A^2)+\frac{1}{q}(\lambda \alpha^2+\frac{1-\lambda}{d}A^2)}.
\label{eq:opt_F_dp}
\end{equation}
Note that $F(\lambda)$ is a quadratic rational function w.r.t. $\lambda$.

Since we optimize $\lambda$ in a closed interval $[0, 1]$, checking $\lambda=0$, $\lambda=1$ and the stationary points (i.e., $F'(\lambda)=0$) would suffice. By solving $F'(\lambda)=0$, we have at most two solutions:
\begin{align}
    \lambda_1&=\frac{A^2(\alpha^2(d+2q-2)-A^2)}{2A^2\alpha^2 (dq+q-1)-\alpha^4 d^2-A^4}, \label{eq:sol_lambda_1_dp} \\
    \lambda_2&=\frac{A^2}{A^2-\alpha^2 d} \nonumber,
\end{align}
where $\lambda_1$ or $\lambda_2$ is the solution if and only if the denominator is not 0. $\lambda_2\notin (0,1)$, so we only need to consider $\lambda_1$.

First, we figure out when $\lambda_1\in (0,1)$. We can verify that $\lambda_1=1$ when $\alpha^2=0$ and $\lambda_1=0$ when $A^2=0$. Suppose $\alpha^2\neq 0$ and $A^2\neq 0$. Let $J$ denote the numerator in Eq.~\eqref{eq:sol_lambda_1_dp} and $K$ denote the denominator. We have that when $\alpha^2> \frac{A^2}{d+2q-2}$, $J> 0$; else, $J\leq 0$. We also have that when $\alpha^2< \frac{A^2(2q-1)}{d}$, $J < K$; else, $J\geq K$. Note that $J/K\in (0,1)$ if and only if $0< J < K$ or $0> J> K$. Hence, $\lambda_1\in (0,1)$ if and only if $\frac{A^2}{d+2q-2} < \alpha^2 < \frac{A^2(2q-1)}{d}$.

Case 1: $\lambda_1\notin (0,1)$. Then it suffices to compare $F(0)$ and $F(1)$. We have
\begin{equation*}
    F(0)=\frac{A^2 q}{d+q-1}, F(1)=\alpha^2.
\end{equation*}
Hence, $F(0)\geq F(1)$ if and only if $\alpha^2\leq \frac{A^2 q}{d+q-1}$. It means that if $\alpha^2\geq\frac{A^2(2q-1)}{d}$, then $\lambda^*=1$; if $\alpha^2\leq\frac{A^2}{d+2q-2}$, then $\lambda^*=0$.

Case 2: $\lambda_1\in (0,1)$. After plugging Eq.~\eqref{eq:sol_lambda_1_dp} into Eq.~\eqref{eq:opt_F_dp}, we have
\begin{align}
    F(\lambda_1)&=\frac{4A^2\alpha^2 (A^2+\alpha^2)(q-1)q}{2A^2\alpha^2(2q(d+q-1)-d)-\alpha^4 d^2-A^4}. \label{eq:F_lambda_1_dp}
\end{align}

Now we prove that $F(\lambda_1)\geq F(0)$ and $F(\lambda_1)\geq F(1)$.
Since when $0<\lambda<1$, both the numerator and the denominator in Eq.~\eqref{eq:opt_F_dp} is positive, we have $F(\lambda)>0$, $\forall \lambda\in (0,1)$. Since the numerator in Eq.~\eqref{eq:F_lambda_1_dp} is non-negative, and $F(\lambda_1)> 0$, we know that the denominator in Eq.~\eqref{eq:F_lambda_1_dp} is positive. Hence, we have
\begin{align*}
    F(\lambda_1)-F(0)&=\frac{q A^2(\alpha^2(d+2q-2)-A^2)^2}{(q+d-1)(2A^2\alpha^2(2q(d+q-1)-d)-\alpha^4 d^2-A^4)}> 0; \\
    F(\lambda_1)-F(1)&=\frac{\alpha^2(\alpha^2 d+A^2(1-2q))^2}{2A^2\alpha^2(2q(d+q-1)-d)-\alpha^4 d^2-A^4}> 0.
\end{align*}
Hence in this case $\lambda^*=\lambda_1$.

The proof is complete.
\end{proof}

\subsection{Explanation on Eq.~(\protect\ref{eq:sampling-dp})}
\label{sec:proof_eq-sampling-dp}
We explain why the construction of $u_i$ in Eq.~\eqref{eq:sampling-dp} makes $\E[u_iu_i^\top]$ a good approximation of $\mathbf{C}$.

Recall the setting: In $\mathbb{R}^D$, we have a normalized transfer gradient $v$, and a specified $d$-dimensional subspace with $\{v_1,...,v_d\}$ as its orthonormal basis. Let $\mathbf{C}=\lambda v v^\top+\frac{1-\lambda}{d}\sum_{i=1}^d v_i v_i^\top$. Here we argue that if $u=\sqrt{\lambda}\cdot v+\sqrt{1-\lambda}\cdot\overline{(\mathbf{I}-vv^\top)\mathbf{V}\xi}$, then $\E[u u^\top]\approx \mathbf{C}$.

Let $r\triangleq \overline{(\mathbf{I}-vv^\top)\mathbf{V}\xi}$. The reason why $\E[u u^\top]\neq \mathbf{C}$ is that $\E[rr^\top]\neq \frac{1}{d}\sum_{i=1}^d v_i v_i^\top$ when $v$ is not orthogonal to the subspace spanned by $\{v_1,...,v_d\}$. However, by symmetry, we still have $\E[r]=0$. To get an expression of $\E[rr^\top]$, we let $v_T$ denotes the projection of $v$ onto the subspace, and let $v_1=\overline{v_T}$ so that $v_2,...,v_d$ are orthonormal to $v_T$ (hence also orthonormal to $v$). We temporarily assume $v_T\neq v$ and $v_T\neq 0$. Now let $v_1'=\overline{(\mathbf{I}-vv^\top)v_T}=\overline{v_T-v^\top v_T\cdot v}$, then $\{v_1',v_2,...,v_d\}$ form an orthonormal basis of the subspace in which $r$ lies, and $v$ is orthogonal to this modified subspace. Now we have $\E[rr^\top]=\lambda_1 v_1'v_1'^\top + \frac{1-\lambda_1}{d-1}\sum_{i=2}^d v_i v_i^\top$ where $\lambda_1$ is a number in $[0, \frac{1}{d}]$. (Note that when $v=v_T$, although $v_1'$ cannot be defined, we have $\lambda_1=0$. When $v_T=0$, we can just set $v_1'=v_1$ and $\lambda_1=\frac{1}{d}$.) When $d$ is large, $\lambda_1$ is small, so for approximation we can replace $v_1'$ with $v_1$; $|\lambda-\frac{1}{d}|$ is small, so for approximation we can set $\lambda_1=\frac{1}{d}$. Then we have $\E[rr^\top]\approx \frac{1}{d}\sum_{i=1}^d v_i v_i^\top$. Since $\E[r]=0$, we have $\E[uu^\top]=\lambda vv^\top + (1-\lambda)\E[rr^\top]\approx \lambda vv^\top + \frac{1-\lambda}{d}\sum_{i=1}^d v_i v_i^\top$.

\begin{remark}
To avoid approximation, one can choose the subspace as spanned by $\{v_1',v_2,...,v_d\}$ instead of $\{v_1,v_2,...,v_d\}$ to ensure that $v$ is orthogonal to the subspace. Then $u$ can be sampled as
\begin{equation*}
    u=\sqrt{\lambda}\cdot v+\sqrt{1-\lambda}\cdot\overline{\mathbf{V}'\xi},
\end{equation*}
where $\mathbf{V}'=[v_1',v_2,...,v_d]$ and $\xi$ is sampled uniformly from the $d$-dimensional unit hypersphere. Note that here the optimal $\lambda$ is calculated using $A'^2=v_1'^\top\overline{\nabla f(x)}+\sum_{i=2}^d (v_i^\top\overline{\nabla f(x)})^2$. However, in practice, it is not convenient to make the subspace dependent on $v$, and the computational complexity is high to construct an orthonormal basis with one vector ($v_1'$) specified.
\end{remark}

\section{Gradient averaging method}
\label{sec:average}
In Sec.~\ref{sec:biased_sampling}, we have presented the prior-guided random gradient-free (P-RGF) algorithm, where we integrate the transfer gradient into the sampling distribution of $u_i$. In this section, we propose the gradient averaging algorithm as an alternative method to incorporate the transfer gradient. The motivation is as follows. We observe that the RGF estimator in Eq.~\eqref{eq:estimate} is in the following form: $\hat{g}=\frac{1}{q}\sum_{i=1}^q\hat{g}_i$, where multiple rough estimates are averaged. Indeed, the transfer gradient itself can also be considered as an estimate of the true gradient, and then it is reasonable to adopt a weighted average of the transfer gradient and the RGF estimator. Here, we choose the RGF estimator to be the ordinary one (using $u_i$ sampled from uniform distribution) instead of the P-RGF estimator, to prevent its direction from being too similar to the direction of the transfer gradient.

In summary, the gradient averaging method works as follows. We first get the RGF estimator denoted by $\hat{g}^U$, given by Eq.~\eqref{eq:estimate} with the sampling distribution $\mathcal{P}$ being the uniform distribution; then normalize the estimator; and finally average the normalized transfer gradient $v$ and the normalized RGF estimator $\overline{\hat{g}^U}$ as
\begin{equation}
    \hat{g} = \mu v+(1-\mu) \overline{\hat{g}^U},
    \label{eq:average}
\end{equation}
where $\mu\in [0,1]$ plays a similar role as $\lambda$ in the proposed prior-guided RGF method. 
We also assume $\alpha=v^\top\overline{\nabla f(x)}\geq 0$.
Under the gradient estimation problem, we also want to minimize $L(\hat{g})$ by optimizing $\mu$.
First, we have the following theorem.

\begin{theorem}\label{thm:average}
    Let $\beta=\overline{\nabla f(x)}^\top\overline{\frac{1}{q}\sum_{i=1}^q (u_i^\top \nabla f(x)\cdot u_i)}$ be the cosine similarity between $\nabla f(x)$ and the ordinary RGF estimator w.r.t. a locally linear $f$. If $f$ is differentiable at $x$, the loss of the gradient estimator in Eq.~\eqref{eq:average} is
    \begin{equation}
        \lim_{\sigma\to 0} L(\hat{g})=(1-\frac{(\mu\alpha+(1-\mu)\E[\beta])^2}{\mu^2+(1-\mu)^2+2\mu(1-\mu)\alpha\E[\beta]})\|\nabla f(x)\|_2^2.
        \label{eq:theorem-2}
    \end{equation}
\end{theorem}

\begin{proof}
As in Eq.~\eqref{eq:estimate}, $\hat{g}^U=\frac{1}{q}\sum_{i=1}^q\hat{g}_i^U$ and $\hat{g}_i^U=\frac{f(x+\sigma u_i)-f(x)}{\sigma}\cdot u_i$. First, we derive $L(\hat{g})$ based on the assumption that $\hat{g}_i^U$ is equal to $u_i^\top\nabla f(x) \cdot u_i$, which will hold when $f$ is locally linear.

\begin{lemma}
    Assume that $\hat{g}^U=\frac{1}{q}\sum_{i=1}^q (u_i^\top \nabla f(x)\cdot u_i)$ (then $\beta=\overline{\nabla f(x)}^\top \overline{\hat{g}^U}$). We have
    \begin{equation*}
        L(\hat{g})=(1-\frac{(\mu\alpha+(1-\mu)\E[\beta])^2}{\mu^2+(1-\mu)^2+2\mu(1-\mu)\alpha\E[\beta]})\|\nabla f(x)\|_2^2.
    \end{equation*}
\end{lemma}

\begin{proof}
It can be verified\footnote{If $\hat{g}^U=0$, $\nabla f(x)^\top \hat{g}^U=\frac{1}{q}\sum_{i=1}^q (u_i^\top \nabla f(x))^2=0$, hence $u_i^\top\nabla f(x)=0$ for $i=1,2,...,q$, whose probability is 0. \label{fn:prob0}} that $\hat{g}^U=0$ happens with probability 0, hence we restrict our consideration to the set $\{\hat{g}^U\neq 0\}$, which does not affect our conclusion. Then $\overline{\hat{g}^U}$ is always well-defined. The distribution of $\hat{g}^U$ is symmetric around the direction of $\nabla f(x)$, and so is the distribution of $\overline{\hat{g}^U}$. Hence we can suppose that $\E[\overline{\hat{g}^U}]=k\overline{\nabla f(x)}$. Since $\E[\beta]=\overline{\nabla f(x)}^\top \E[\overline{\hat{g}^U}]=k$, we have $\E[\overline{\hat{g}^U}]=\E[\beta]\overline{\nabla f(x)}$.

Hence we have
\begin{align*}
    \nabla f(x)^\top\E[\overline{\hat{g}^U}] = \nabla f(x)^\top\E[\beta]\overline{\nabla f(x)} = \E[\beta]\|\nabla f(x)\|_2,
\end{align*}
and
\begin{align*}
    v^\top\E[\overline{\hat{g}^U}]=v^\top\E[\beta]\overline{\nabla f(x)}=\alpha\E[\beta].
\end{align*}
Together with $v^\top \nabla f(x)=\alpha\|\nabla f(x)\|_2$ and noting that $\|v\|_2=1$, we have
\begin{align}
\E\|\nabla f(x)-b\hat{g}\|_2^2 &= \E\|b\mu v + b(1-\mu)\overline{\hat{g}^U} - \nabla f(x)\|^2 \nonumber \\
&= b^2\mu^2 + b^2(1-\mu)^2 + \|\nabla f(x)\|^2_2 + 2b^2\mu(1-\mu)v^\top\E[ \overline{\hat{g}^U}] \nonumber \\ &\ \ \ \ - 2b\mu\alpha\|\nabla f(x)\|_2 - 2b(1-\mu)\nabla f(x)^\top\E[\overline{\hat{g}^U}] \label{eq:average_proof} \\
&= b^2\mu^2 + b^2(1-\mu)^2 + \|\nabla f(x)\|^2_2 + 2b^2 \mu(1-\mu) \alpha \E[\beta] \nonumber \\ &\ \ \ \ - 2b\mu\alpha\|\nabla f(x)\|_2 - 2b(1-\mu)\E[\beta]\|\nabla f(x)\| \nonumber \\ &=((1-\mu)^2+\mu^2+2\mu(1-\mu)\alpha \E[\beta])b^2 \nonumber \\ &\ \ \ \ -2(\alpha \mu+\E[\beta] (1-\mu))\|\nabla f(x)\|_2 b + \|\nabla f(x)\|^2_2. \nonumber
\end{align}
Since $\nabla f(x)^\top \hat{g}^U=\frac{1}{q}\sum_{i=1}^q (u_i^\top \nabla f(x))^2\geq 0$, then $\beta\geq 0$, and hence $\E[\beta]\geq 0$. Then $(1-\mu)^2+\mu^2+2\mu(1-\mu)\alpha \E[\beta]>0$ and $\alpha \mu+\E[\beta] (1-\mu)\geq 0$. Since $L(\hat{g})=\min_{b\geq 0}\E\|\nabla f(x)-b\hat{g}\|_2^2$, optimize the objective w.r.t. $b$ and we complete the proof.
\end{proof}

Next, we prove that if $f$ is not locally linear, as long as it is differentiable at $x$, then by picking a sufficient small $\sigma$, the loss tends to be that of the local linear approximation. Here, we redefine the notation as follows. We make the dependency of $\hat{g}^U$ on $\sigma$ explicit, i.e. we use $\hat{g}^U_\sigma$ to denote it. Meanwhile, we define $\hat{g}^U_0\triangleq \frac{1}{q}\sum_{i=1}^q (u_i^\top \nabla f(x)\cdot u_i)$ as the RGF estimator under the local linear approximation. We define $\hat{g}_\sigma=\mu v+(1-\mu)\hat{g}^U_\sigma$ and $\hat{g}_0=\mu v+(1-\mu)\hat{g}^U_0$. Then we have
\begin{lemma}
    If $f$ is differentiable at $x$, then
    \begin{equation*}
        \lim_{\sigma\to 0}L(\hat{g}_\sigma)=L(\hat{g}_0)
    \end{equation*}
\end{lemma}
\begin{proof}
By Eq.~\eqref{eq:average_proof}, it suffices to prove $\lim_{\sigma\to 0}\E[\overline{\hat{g}^U_\sigma}]=\E[\overline{\hat{g}^U_0}]$.

For any value of $u_1,u_2,...,u_q$, we have $\lim_{\sigma\to 0}\hat{g}^U_\sigma=\hat{g}^U_0$, i.e. $\hat{g}^U_\sigma$ converges pointwise to $\hat{g}^U_0$. Recall that $\mathrm{Pr}(\hat{g}^U_0=0)=0$, so we can restrict our consideration to the set $\{\hat{g}^U_0\neq 0\}$ which does not affect our conclusion. Since $\overline{x}=\frac{x}{\|x\|_2}$ is continuous everywhere in its domain, $\overline{\hat{g}^U_\sigma}$ converges pointwise to $\overline{\hat{g}^U_0}$. Since the family $\{\overline{\hat{g}^U_\sigma}\}$ is uniformly bounded, by the dominated convergence theorem we have $\lim_{\sigma\to 0}\E[\overline{\hat{g}^U_\sigma}]=\E[\overline{\hat{g}^U_0}]$.
\end{proof}
By combining the two lemmas above, our proof for the theorem is complete.
\end{proof}
We can calculate the closed-form solution of $\mu^*$, the value of $\mu$ minimizing Eq.~\eqref{eq:theorem-2}, as
\begin{equation}
    \mu^*=\frac{\alpha(1-\E[\beta]^2)}{(1-\alpha^2)\E[\beta]+\alpha(1-\E[\beta]^2)}\approx \frac{\alpha}{\E[\beta]+\alpha}.
    \label{eq:mu-1}
\end{equation}
That is, the ratio of weights of $v$ and $\overline{\hat{g}^U}$ is approximately the ratio of their (expected) inner product with the true gradient.

Next, we discuss how to calculate $\E[\beta]=\E[\overline{\nabla f(x)}^\top\overline{\hat{g}^U_0}]$, where $\hat{g}^U_0=\frac{1}{q}\sum_{i=1}^q (u_i^\top \nabla f(x)\cdot u_i)$. $\E[\beta]$ is independent of $\|\nabla f(x)\|_2$, and since $u_i$ is uniformly sampled from the unit hypersphere, $\E[\beta]$ is also independent of the direction of $\nabla f(x)$. Hence, $\E[\beta]$ is a constant given the dimension $D$ and the number of queries $q$, and we can estimate $\E[\beta]$ using numerical simulation methods.

However, here we give a framework for approximating $\E[\beta]$ in a closed-form formula. We notice that the following approximation works well in practice, where $\hat{g}=\frac{1}{q}\sum_{i=1}^q (u_i^\top \nabla f(x)\cdot u_i)$:
\begin{align*}
    \E[\beta]&=\E[\sqrt{\beta^2}] \\&\approx \sqrt{\E[\beta^2]}\\&=\sqrt{1-\E[\min_b\|\overline{\nabla f(x)}-b\hat{g}\|^2]}\\&=\sqrt{1-\frac{1}{\|\nabla f(x)\|_2^2}\E[\min_b\|\overline{\nabla f(x)}-b\hat{g}\|^2]}\\&\approx \sqrt{1-\frac{1}{\|\nabla f(x)\|_2^2}\min_b\E\|\overline{\nabla f(x)}-b\hat{g}\|^2}\\&=\sqrt{1-\frac{1}{\|\nabla f(x)\|_2^2}L(\hat{g})^2}.
\end{align*}
Here, the first equality is because $\nabla f(x)^\top \hat{g}=\frac{1}{q}\sum_{i=1}^q (u_i^\top \nabla f(x))^2\geq 0$; the second equality is because we have $\min_b\|\overline{\nabla f(x)}-b\hat{g}\|^2=1-(\overline{\nabla f(x)}^\top \overline{\hat{g}})^2=1-\beta^2$. Intuitively, the two approximations work well because the variances of $\beta$ and $\|\hat{g}\|_2$ are relatively small.

Now we define $F(\hat{g})=1-\frac{1}{\|\nabla f(x)\|_2^2}L(\hat{g})^2$. Then we have $\E[\beta]\approx \sqrt{F(\hat{g})}$. Note that when $u_i$ is sampled from the uniform distribution on the unit hypersphere, $F(\hat{g})$ is in fact $F(\frac{1}{D})$ in Eq.~\eqref{eq:opt_F}, since $\hat{g}$ is an RGF estimator w.r.t. locally linear $f$, and $\E[u_i u_i^\top]=\frac{1}{D}\mathbf{I}$ which corresponds to $\lambda=\frac{1}{D}$ in Eq.~\eqref{eq:mix_C}. We can calculate $F(\frac{1}{D})=\frac{q}{D+q-1}$. Hence, $\E[\beta]\approx\sqrt{\frac{q}{D+q-1}}$.

Calculating $\mu^*$ using $\alpha\geq 0$ and $\E[\beta]\approx \sqrt{\frac{q}{D+q-1}}>0$, we have $\mu^*<1$. This means we always need to take $q$ queries to get $\overline{\hat{g}^U}$. However, when $\mu$ is close to $1$, the improvement of using $\hat{g}=\mu^* v+(1-\mu^*)\overline{\hat{g}^U}$ instead of directly using $v$ as the estimate is marginal. To save queries, we adopt a threshold $c\in (0,1)$. When $\mu^*\geq c$, we let $\hat{g}=v$ instead of letting $\hat{g}=\mu^*v+(1-\mu^*)\overline{\hat{g}^U}$.
\begin{algorithm}[t]
\small
\caption{Gradient averaging method}
\label{alg:average}
\begin{algorithmic}[1]
\Require The black-box model $f$; input $x$ and label $y$; the normalized transfer gradient $v$; sampling variance $\sigma$; number of queries $q$; input dimension $D$; threshold $c$.
\Ensure Estimate of the gradient $\nabla f(x)$.
\State Estimate the cosine similarity $\alpha=v^\top\overline{\nabla f(x)}$ (detailed in Sec.~\ref{sec:alpha});
\State Approximate $\E[\beta]$ as $\sqrt{\frac{q}{D+q-1}}$;
\State Calculate $\mu^*$ according to Eq.~\eqref{eq:mu-1} given $\alpha$ and $\E[\beta]$;
\If {$\mu^*\geq c$}
\Return $v$;
\EndIf
\State $\hat{g}^U \leftarrow \mathbf{0}$;
\For {$i = 1$ to $q$}
\State Sample $u_i$ from the uniform distribution on the $D$-dimensional unit hypersphere;
\State $\hat{g}^U \leftarrow \hat{g}^U + \dfrac{f(x + \sigma u_i,y) - f(x,y)}{\sigma} \cdot u_i$;
\EndFor
\Return $\nabla f(x)\leftarrow \mu^* v+(1-\mu^*)\overline{\hat{g}^U}$.
\end{algorithmic}
\end{algorithm}

We summarize the gradient averaging method in Algorithm~\ref{alg:average}.

\subsection{Incorporating the data-dependent prior}
We can also incorporate the data-dependent prior introduced in Sec.~\ref{sec:dp} into the proposed gradient averaging method. In this case, we get an ordinary subspace RGF estimate $\hat{g}^S$ first\footnote{An ordinary subspace RGF estimate refers to the RGF estimate in Eq.~\eqref{eq:estimate} with $u_i=\mathbf{V}\xi_i$, where $\xi_i$ is sampled uniformly from the $d$-dimensional unit hypersphere, $\mathbf{V}=[v_1,v_2,...,v_d]$, and $\{v_1, v_2, ..., v_d\}$ is an orthonormal basis of a $d$-dimensional subspace. It corresponds to $\lambda=0$ in Eq.~\eqref{eq:mix_dp_C}.} (instead of an ordinary RGF estimate); and then normalize it; and finally get the averaged estimator as
\begin{equation}
    \hat{g} = \mu v+(1-\mu)\overline{\hat{g}^S}.
    \label{eq:dp_average}
\end{equation}
We also assume $\alpha=v^\top\overline{\nabla f(x)}\geq 0$. Here, we need to analyze some quantity about the subspace. We define $\overline{\nabla f(x)}_T=(\sum_{i=1}^d v_i v_i^\top)\overline{\nabla f(x)}$ is the projection of $\overline{\nabla f(x)}$ into the subspace corresponding to the data-dependent prior, and $A^2=\sum_{i=1}^d (v_i^\top \overline{\nabla f(x)})^2=\|\overline{\nabla f(x)}_T\|^2$. Then we have the following loss function:
\begin{theorem}
Let $\beta=\overline{\nabla f(x)}^\top\overline{\frac{1}{q}\sum_{i=1}^q (u_i^\top \nabla f(x)\cdot u_i)}$ be the cosine similarity between $\nabla f(x)$ and the ordinary subspace RGF estimator w.r.t. a locally linear $f$. (Note that here $u_i$ lies in the subspace.) Furthermore, let $\alpha_1=v^\top\overline{\nabla f(x)}_T$. If $f$ is differentiable at $x$ and $A^2>0$, using $\hat{g}$ defined in \eqref{eq:dp_average}, we have
\begin{equation}
    \lim_{\sigma\to 0} L(\hat{g})=(1-\frac{(\mu\alpha+(1-\mu)\E[\beta])^2}{\mu^2+(1-\mu)^2+2\mu(1-\mu)\frac{\alpha_1}{A^2}\E[\beta]})\|\nabla f(x)\|^2.
    \label{eq:theorem-3}
\end{equation}
\end{theorem}
\begin{proof}
Similar to the proof of Theorem~\ref{thm:average}, we define $\hat{g}^S_0=\frac{1}{q}\sum_{i=1}^q (u_i^\top \nabla f(x)\cdot u_i)=\frac{1}{q}\sum_{i=1}^q (u_i^\top \nabla f(x)_T\cdot u_i)$, where $\nabla f(x)_T=\|\nabla f(x)\|_2\overline{\nabla f(x)}_T$ denotes the projection of $\nabla f(x)$ into the subspace. Then $\beta=\overline{\nabla f(x)}^\top \overline{\hat{g}^S_0}=\overline{\nabla f(x)}_T^\top \overline{\hat{g}^S_0}$. Since $A^2>0$, we have $\nabla f(x)_T\neq 0$, hence as described in Footnote~\ref{fn:prob0}, we can prove $\mathrm{Pr}(\hat{g}^S_0=0)=0$ similarly. Now we restrict our consideration to the set $\{\hat{g}^S_0\neq 0\}$. The distribution of $\hat{g}^S_0$ is symmetric around the direction of $\nabla f(x)_T$, and so is the distribution of $\overline{\hat{g}^S_0}$. Hence we can suppose that $\E[\overline{\hat{g}^S_0}]=k\overline{\nabla f(x)}_T$. Since $\E[\beta]=\overline{\nabla f(x)}_T^\top \E[\overline{\hat{g}^S_0}]=k\|\overline{\nabla f(x)}_T\|_2^2=kA^2$, we have $\E[\overline{\hat{g}_0^S}]=\frac{\E[\beta]}{A^2}\overline{\nabla f(x)}_T$.

Note that
\begin{align*}
    v^\top\E[\overline{\hat{g}_0^S}]=v^\top\frac{\E[\beta]}{A^2}\overline{\nabla f(x)}_T=\frac{\alpha_1}{A^2}\E[\beta].
\end{align*}
The rest of the proof is the same as that of Theorem~\ref{thm:average}.
\end{proof}
The optimal solution of $\mu$ minimizing Eq.~\eqref{eq:theorem-3} is
\begin{equation}
    \mu^*=\frac{A^2\alpha-\alpha_1\E[\beta]^2}{(A^2-\alpha_1\E[\beta])(\alpha+\E[\beta])}\approx \frac{\alpha}{\E[\beta]+\alpha}.
    \label{eq:mu-2}
\end{equation}
The approximation works mainly because $A\gg \E[\beta]$ (since $\E[\beta]\approx A\sqrt{\frac{q}{d+q-1}}$ as shown below). Hence, the approximate solution is the same as in the case without using the data-dependent prior, which does not depend on $\alpha_1$.

Similarly, we can approximate $\E[\beta]$ by $\E[\beta]\approx \sqrt{F(\hat{g})}$. When $u_i$ is sampled from the uniform distribution on the unit hypersphere in the subspace, $F(\hat{g})$ is in fact $F(0)$ in Eq.~\eqref{eq:opt_F_dp}, since $\hat{g}$ is an RGF estimator w.r.t. locally linear $f$, and $\E[u_iu_i^T]=\frac{1}{d}\sum_{i=1}^d v_i v_i^\top$ which corresponds to $\lambda=0$ in Eq.~\eqref{eq:mix_dp_C}. We can calculate $F(0)=\frac{A^2q}{d+q-1}$. Hence, $\E[\beta]\approx \sqrt{\frac{A^2q}{d+q-1}}$.

Our gradient averaging algorithm with the data-dependent prior is similar to Algorithm~\ref{alg:average}.
We first estimate $\alpha$ and $A$, approximate $\E[\beta]$ as $\sqrt{\frac{A^2q}{d+q-1}}$, and then calculate $\mu^*$ by Eq.~\eqref{eq:mu-2}.
If $\mu^*\geq c$, we use the transfer gradient $v$ as the estimate.
If not, we get the ordinary subspace RGF estimator $\hat{g}^S$, then use $\hat{g}\leftarrow \mu^* v+(1-\mu^*)\overline{\hat{g}^S}$ as the estimate.

\section{Estimation of $A$}
\label{sec:estimation-A}
Suppose that the subspace is spanned by a set of orthonormal vectors $\{v_1,...,v_d\}$. Now we want to estimate 
\begin{align*}
    A^2=\sum_{i=1}^d (v_i^\top \overline{\nabla f(x)})^2=\frac{\sum_{i=1}^d (v_i^\top \nabla f(x))^2}{\|\nabla f(x)\|_2^2}=\frac{\|h(x)\|_2^2}{\|\nabla f(x)\|_2^2},
\end{align*}
where $h(x)=\sum_{i=1}^d v_i^\top\nabla f(x)\cdot v_i$ is the projection of $\nabla f(x)$ to the subspace. We can estimate $\|\nabla f(x)\|_2^2$ using the method introduced in Sec.~3.3. Here, we introduce the method to estimate $\|h(x)\|_2^2$.

Let $w=\mathbf{V}\xi$ where $\mathbf{V}=[v_1,v_2,...,v_d]$ and $\xi$ is a random vector uniformly sampled from the $d$-dimensional unit hypersphere. By Lemma~\ref{lem:covariance}, $\E[ww^\top]=\frac{1}{d}\sum_{i=1}^d v_i v_i^\top$. Suppose we have $S$ i.i.d. such samples of $w$ denoted by $w_1, ..., w_S$, and we let $\mathbf{W}=[w_1, ..., w_S]$.

With $g(x_1, ..., x_S)=\frac{1}{S}\sum_{s=1}^S x_s^2$, we have
\begin{align*}
    g(\mathbf{W}^\top\nabla f(x))=g(\mathbf{W}^\top h(x))=\|h(x)\|_2^2\cdot g(\mathbf{W}^\top \overline{h(x)}).
\end{align*}
Hence $\frac{g(\mathbf{W}^\top\nabla f(x))}{\E[g(\mathbf{W}^\top \overline{h(x)})]}$ is an unbiased estimator of $\|h(x)\|_2^2$. Now, $\overline{h(x)}$ is in the subspace spanned by $\{v_1,...,v_d\}$, and $w_1$ is uniformly distributed on the unit hypersphere of this subspace. Hence $\E[(w_1^\top\overline{h(x)})^2]$ is independent of the direction of $\overline{h(x)}$ and can be computed. We have:
\begin{align*}
    \E[g(\mathbf{W}^\top \overline{h(x)})]=\E[(w_1^\top\overline{h(x)})^2]=\overline{h(x)}^\top\E[w_1 w_1^\top]\overline{h(x)}=\overline{h(x)}^\top\frac{1}{d}\sum_{i=1}^d v_i v_i^\top\overline{h(x)}=\frac{1}{d}.
\end{align*}
Hence, we have the estimator $\|h(x)\|_2\approx \sqrt{\frac{d}{S}\sum_{s=1}^S(w_s^\top\nabla f(x))^2}$, where $w_s=\mathbf{V}\xi_s$ and $\xi_s$ is uniformly sampled from the unit hypersphere in $\mathbb{R}^d$. Finally we can get an estimate of $A$ by $A=\frac{\|h(x)\|_2}{\|\nabla f(x)\|_2}$.

\section{Additional experiments}
\label{sec:additional-exps}
We add the experimental results using the gradient averaging method, including a baseline method which uses a fixed $\mu$ set to $0.5$ or $0.05$ and the algorithm using the optimal value $\mu^*$ given by Eq.~\eqref{eq:mu-1} (or by Eq.~\eqref{eq:mu-2} in the case with the data-dependent prior). We set $c=\frac{1}{1+\sqrt{2}}$, and the other hyperparameters are the same with those for the P-RGF method. Table~\ref{tab:final-results-full} and Table~\ref{tab:final-results-defense-full} are the full tables of experimental results based on the $\ell_2$ norm.

We show the experimental results based on the $\ell_\infty$ norm in this section. We set the perturbation budget as $\epsilon=0.05$, the step size as $\eta=0.005$ in the PGD method. Other hyperparameters are the same with those for $\ell_2$ attacks.
Table~\ref{tab:final-results-linfty} and Table~\ref{tab:final-results-linfty-defense} show the results for attacking the normal models and the defensive models, respectively.
Our method also leads to better results, which are consistent with those based on the $\ell_2$ norm.

\begin{table}
  \caption{The full experimental results of black-box attacks against Inception-v3, VGG-16, and ResNet-50 under the $\ell_2$ norm. We report the attack success rate (ASR) and the average number of queries (AVG. Q) needed to generate an adversarial example over successful attacks.}
  \label{tab:final-results-full}
  \centering

  \begin{tabular}{l|cc|cc|cc}
    \hline
    \multirow{2}{*}{Methods} & \multicolumn{2}{c|}{Inception-v3} & \multicolumn{2}{c|}{VGG-16} & \multicolumn{2}{c}{ResNet-50}\\
    \cline{2-7}
    & ASR & AVG. Q & ASR & AVG. Q & ASR & AVG. Q \\
    \hline
    NES~\cite{ilyas2018black} & 95.5\% & 1718 & 98.7\% & 1081 & 98.4\% & 969 \\
    Bandits\textsubscript{T}~\cite{ilyas2018prior} & 92.4\% & 1560 & 94.0\% & 584 & 96.2\% & 1076 \\
    Bandits\textsubscript{TD}~\cite{ilyas2018prior} & 97.2\% & 874 & 94.9\% & 278 & 96.8\% & 512 \\
    AutoZoom~\cite{tu2018autozoom} & 85.4\% & 2443 & 96.2\% & 1589 & 94.8\% & 2065 \\
    \hline
    RGF & 97.7\% & 1309 & \bf99.8\% & 749 & \bf99.6\% & 673 \\
    P-RGF ($\lambda=0.5$) & 96.5\% & 1119 & 97.8\% & 710 & 98.7\% & 635 \\
    P-RGF ($\lambda=0.05$) & 97.8\% & 1021 & 99.7\% & 624 & 99.3\% & 511 \\
    P-RGF ($\lambda^*$) & \bf98.1\% & 745 & 99.6\% & 331 & \bf99.6\% & 265 \\
    Averaging ($\mu=0.5$) & 97.9\% & 958 & \bf99.8\% & 528 & \bf99.6\% & 485 \\
    Averaging ($\mu=0.05$) & 97.8\% & 1260 & \bf99.8\% & 700 & \bf99.6\% & 619 \\
    Averaging ($\mu^*$) & 97.9\% & \bf735 & 99.7\% & \bf320 & 99.5\% & \bf250 \\
    \hline
    RGF\textsubscript{D} & 99.1\% & 910 & \bf100.0\% & 372 & 99.7\% & 429 \\
    P-RGF\textsubscript{D} ($\lambda=0.5$) & 98.2\% & 1047 & 99.7\% & 634 & 99.5\% & 552 \\
    P-RGF\textsubscript{D} ($\lambda=0.05$) & 99.1\% & 754 & 99.9\% & 359 & \bf99.8\% & 379 \\
    P-RGF\textsubscript{D} ($\lambda^*$) & 99.1\% & 649 & 99.8\% & 250 & 99.6\% & \bf232 \\
    Averaging\textsubscript{D} ($\mu=0.5$) & \bf99.3\% & 734 & \bf100.0\% & 332 & 99.7\% & 340 \\
    Averaging\textsubscript{D} ($\mu=0.05$) & 99.0\% & 865 & \bf100.0\% & 360 & 99.7\% & 404 \\
    Averaging\textsubscript{D} ($\mu^*$) & 99.2\% & \bf644 & 99.7\% & \bf239 & 99.7\% & 240 \\
    \hline
  \end{tabular}
\end{table}

\begin{table}
  \caption{The full experimental results of black-box attacks against JPEG compression~\cite{Guo2017Countering}, randomization~\cite{Xie2018Mitigating}, and guided denoiser~\cite{Liao2017Defense} under the $\ell_2$ norm. We report the attack success rate (ASR) and the average number of queries (AVG. Q) needed to generate an adversarial example over successful attacks.}
  \label{tab:final-results-defense-full}
  \centering
  
  \begin{tabular}{l|cc|cc|cc}
    \hline
    \multirow{2}{*}{Methods} & \multicolumn{2}{c|}{JPEG Compression~\cite{Guo2017Countering}} & \multicolumn{2}{c|}{Randomization~\cite{Xie2018Mitigating}} & \multicolumn{2}{c}{Guided Denoiser~\cite{Liao2017Defense}}\\
    \cline{2-7}
    & ASR & AVG. Q & ASR & AVG. Q & ASR & AVG. Q \\
    \hline
    NES~\cite{ilyas2018black} & 47.3\% & 3114 & 23.2\% & 3632 & 48.0\% & 3633 \\
    SPSA~\cite{uesato2018adversarial} & 40.0\% & 2744 & 9.6\% & 3256 & 46.0\% & 3526 \\
    \hline
    RGF & 41.5\% & 3126 & 19.5\% & 3259 & 50.3\% & 3569 \\
    P-RGF & 61.4\% & 2419 & 60.4\% & 2153 & 51.4\% & 2858 \\
    Averaging & \bf69.4\% & \bf2134 & \bf72.8\% & \bf1739 & \bf66.6\% & \bf2441 \\
    \hline
    RGF\textsubscript{D} & 70.4\% & 2828 & 54.9\% & 2819 & 83.7\% & 2230 \\
    P-RGF\textsubscript{D} & \bf81.1\% & 2120 & \bf82.3\% & 1816 & \bf89.6\% & 1784 \\
    Averaging\textsubscript{D} & 80.6\% & \bf2087 & 77.4\% & \bf1700 & 87.2\% & \bf1777 \\
    \hline
  \end{tabular}
\end{table}

\begin{table}
  \caption{The experimental results of black-box attacks against Inception-v3, VGG-16, and ResNet-50 under the $\ell_\infty$ norm. We report the attack success rate (ASR) and the average number of queries (AVG. Q) needed to generate an adversarial example over successful attacks.}
  \label{tab:final-results-linfty}
  \centering
  
  \begin{tabular}{l|cc|cc|cc}
    \hline
    \multirow{2}{*}{Methods} & \multicolumn{2}{c|}{Inception-v3} & \multicolumn{2}{c|}{VGG-16} & \multicolumn{2}{c}{ResNet-50}\\
    \cline{2-7}
    & ASR & AVG. Q & ASR & AVG. Q & ASR & AVG. Q \\
    \hline
    NES~\cite{ilyas2018black} &  87.5\% & 1850 & 95.6\% & 1477 & 94.5\% & 1405 \\
    Bandits\textsubscript{T}~\cite{ilyas2018prior} & 89.5\% & 1891 & 93.8\% & 585 & 95.2\% & 1199\\
    Bandits\textsubscript{TD}~\cite{ilyas2018prior} & 94.7\% & 1099 & 95.1\% & 288 & 96.5\% & 651 \\
    \hline
    RGF & 94.4\% & 1565 & 98.8\% & 1064 & \bf99.4\% & 990 \\
    P-RGF ($\lambda=0.5$) & 85.4\% & 1578 & 92.2\% & 1180 & 95.1\% & 1046 \\
    P-RGF ($\lambda=0.05$) & 92.7\% & 1409 & 97.5\% & 1131 & 98.3\% & 891 \\
    P-RGF ($\lambda^*$) & 93.8\% & 979 & 98.5\% & 635 & 99.0\% & 507 \\
    Averaging ($\mu=0.5$) & \bf94.9\% & 1263 & 98.9\% & 851 & 99.2\% & 758 \\
    Averaging ($\mu=0.05$) & 94.5\% & 1417 & \bf99.2\% & 1008 & \bf99.4\% & 869 \\
    Averaging ($\mu^*$) & 94.8\% & \bf974 & 98.5\% & \bf560 & 99.3\% & \bf490 \\
    \hline
    RGF\textsubscript{D} & 97.2\% & 1034 & \bf100.0\% & 502 & 99.7\% & 595 \\
    P-RGF\textsubscript{D} ($\lambda=0.5$) & 91.2\% & 1403 & 98.0\% & 1008 & 97.3\% & 852 \\
    P-RGF\textsubscript{D} ($\lambda=0.05$) & 97.7\% & 1005 & 99.9\% & 543 & 99.7\% & 598 \\
    P-RGF\textsubscript{D} ($\lambda^*$) & 97.3\% & 812 & 99.7\% & \bf370 & 99.6\% & 388 \\
    Averaging\textsubscript{D} ($\mu=0.5$) & 98.0\% & 898 & \bf100.0\% & 481 & \bf99.8\% & 504 \\
    Averaging\textsubscript{D} ($\mu=0.05$) & 97.5\% & 930 & \bf100.0\% & 482 & 99.7\% & 548 \\
    Averaging\textsubscript{D} ($\mu^*$) & \bf98.4\% & \bf772 & 99.7\% & 374 & 99.6\% & \bf365 \\
    \hline
  \end{tabular}
\end{table}

\begin{table}
  \caption{The experimental results of black-box attacks against JPEG compression~\cite{Guo2017Countering}, randomization~\cite{Xie2018Mitigating}, and guided denoiser~\cite{Liao2017Defense} under the $\ell_\infty$ norm. We report the attack success rate (ASR) and the average number of queries (AVG. Q) needed to generate an adversarial example over successful attacks.}
  \label{tab:final-results-linfty-defense}
  \centering
  
  \begin{tabular}{l|cc|cc|cc}
    \hline
    \multirow{2}{*}{Methods} & \multicolumn{2}{c|}{JPEG Compression~\cite{Guo2017Countering}} & \multicolumn{2}{c|}{Randomization~\cite{Xie2018Mitigating}} & \multicolumn{2}{c}{Guided Denoiser~\cite{Liao2017Defense}}\\
    \cline{2-7}
    & ASR & AVG. Q & ASR & AVG. Q & ASR & AVG. Q \\
    \hline
    NES~\cite{ilyas2018black} & 29.9\% & 2694 & 14.8\% & 3027 & 20.0\% & 3423 \\
    SPSA~\cite{uesato2018adversarial} & 37.1\% & 2775 & 10.7\% & 2809 & 26.9\% & 3343 \\
    \hline
    RGF & 27.1\% & 2716 & 12.6\% & 3005 & 26.0\% & 3120 \\
    P-RGF & 44.8\% & 2491 & 41.7\% & 2132 & 32.9\% & 2507 \\
    Averaging & \bf51.8\% & \bf2138 & \bf51.9\% & \bf1813 & \bf38.7\% & \bf2251 \\
    \hline
    RGF\textsubscript{D} & 53.4\% & 2708 & 42.4\% & 2444 & 73.3\% & 2158 \\
    P-RGF\textsubscript{D} & \bf64.0\% & 2189 & \bf66.9\% & 2108 & 76.0\% & \bf1799 \\
    Averaging\textsubscript{D} & \bf64.0\% & \bf2141 & 58.3\% & \bf1753 & \bf77.6\% & 1889 \\
    \hline
  \end{tabular}
\end{table}

\end{document}